%% file: tpami_fedgamma.tex
\documentclass[lettersize,journal]{IEEEtran}
\usepackage{amsmath,amsfonts,amssymb}
\usepackage{multirow}
\usepackage{algorithmic}
\usepackage{algorithm}
\usepackage{array}
\usepackage[caption=false,font=normalsize,labelfont=sf,textfont=sf]{subfig}
\usepackage{textcomp}
\usepackage{stfloats}
\usepackage{url}
\usepackage{verbatim}
\usepackage{graphicx}
\usepackage{booktabs}
\usepackage{makecell}
\usepackage{cite}
\usepackage{hyperref}
\usepackage{pifont}
\usepackage{tikz}
\usetikzlibrary{positioning,arrows.meta,shapes.geometric,backgrounds,calc,fit}
\usepackage{amsthm}
\newtheorem{theorem}{Theorem}

\theoremstyle{definition}

\theoremstyle{remark}

\providecommand{\cellsdwidth}{0.85cm}
\providecommand{\stdfont}{\tiny}
\providecommand{\cellsdvskip}{1.4mm}
\providecommand{\cellsdpad}{0.9ex}
\newcommand{\cellsd}[2]{\parbox[c]{\cellsdwidth}{\centering \rule{0pt}{\cellsdpad} #1 \\[\cellsdvskip] \stdfont $\pm$ #2 \rule[-\cellsdpad]{0pt}{\cellsdpad}}}

\begin{document}

\title{Toward Federated Multimodal Graph Foundation Models: A Topology-Aware Multimodal Alignment Framework}

\author{Xunkai Li,
        Guohao Fu,
        Yuming Ai,
        Zhengyu Wu,
        Hongchao Qin,
        Rong-Hua Li,
        Guoren Wang
\thanks{Xunkai Li, Guohao Fu, Yuming Ai, Zhengyu Wu, Hongchao Qin, Rong-Hua Li, and Guoren Wang are with Beijing Institute of Technology, Beijing, 100811, China. (e-mail:cs.xunkai.li@gmail.com; lenfu674@gmail.com; 3120251027@bit.edu.cn; jeremywzy96@outlook.com; qhc.neu@gmail.com; lironghuabit@126.com; wanggrbit@126.com)}

}

\markboth{IEEE Transactions on Pattern Analysis and Machine Intelligence,~Vol.~XX, No.~X, Month~Year}%
{Anonymous \MakeLowercase{\textit{et al.}}: FedGAMMA: Federated Graph and Multimodal Alignment}


\maketitle

\begin{abstract}
Multimodal-attributed graphs (MAGs), whose nodes carry modalities such as images and text alongside topological structure, now pervade applications including social platforms, e-commerce, and biomedical networks, offering richer semantic signals than single-modality graphs.
In practice, such graphs are fragmented across privacy-restricted silos owned by different platforms and institutions, so learning a broadly transferable model over them demands collaborative training that never exposes raw data.
This places the task at the intersection of multimodal graph learning and federated learning, yet existing methods cover only one side of it.
Within each client, an effective model must fuse the multiple modalities on each node with the surrounding graph topology into one coherent representation. Across clients, graphs come from heterogeneous domains with widely varying modality and structure distributions, so naively averaging local models entangles incompatible cross-domain knowledge.
To address the challenges from these two perspectives, we propose \textbf{FedGAMMA} (\underline{\textbf{Fed}}erated \underline{\textbf{G}}raph \underline{\textbf{A}}nd \underline{\textbf{M}}ulti\underline{\textbf{M}}odal \underline{\textbf{A}}lignment), casting federated multimodal graph foundation learning as a two-stage semantic-structural alignment problem of federated pre-training and prompt-based fine-tuning. During pre-training, a shared-private semantic enhancer disentangles cross-modal commonality from modality-specific information, aligning it through optimal transport, a topology-aware graph fusion module decouples semantic and structural views via semantic residual graphs and dual positional encodings, and a dual-channel affinity-aware aggregation mechanism estimates client similarity from feature and graph centroids without exposing raw data. During fine-tuning, FedGAMMA adapts the pretrained encoder through lightweight graph-aware prompts, a shared prompt pool with controlled exploration, and channel-wise prompt synchronization.
Experiments on twelve multimodal graph datasets show FedGAMMA consistently surpassing a broad range of baselines across downstream tasks, with gains of up to $12.96\%$. FedGAMMA further outperforms competitive baselines accross multi-domain datasets on multiple tasks with up to $5.71\%$ under few-shot learning scenario.
\end{abstract}


\input{sections/introduction.tex}

\input{sections/preliminaries.tex}

\input{sections/related_works.tex}

\input{sections/empirical_investigation.tex}

\input{sections/method.tex}

\input{sections/theory.tex}

\input{sections/experiments.tex}

\input{sections/conclusion.tex}


\bibliographystyle{IEEEtran}
\bibliography{fedgamma_refs}


\vspace{-1cm}
\begin{IEEEbiography}[{\includegraphics[width=1in,height=1in,clip,keepaspectratio]{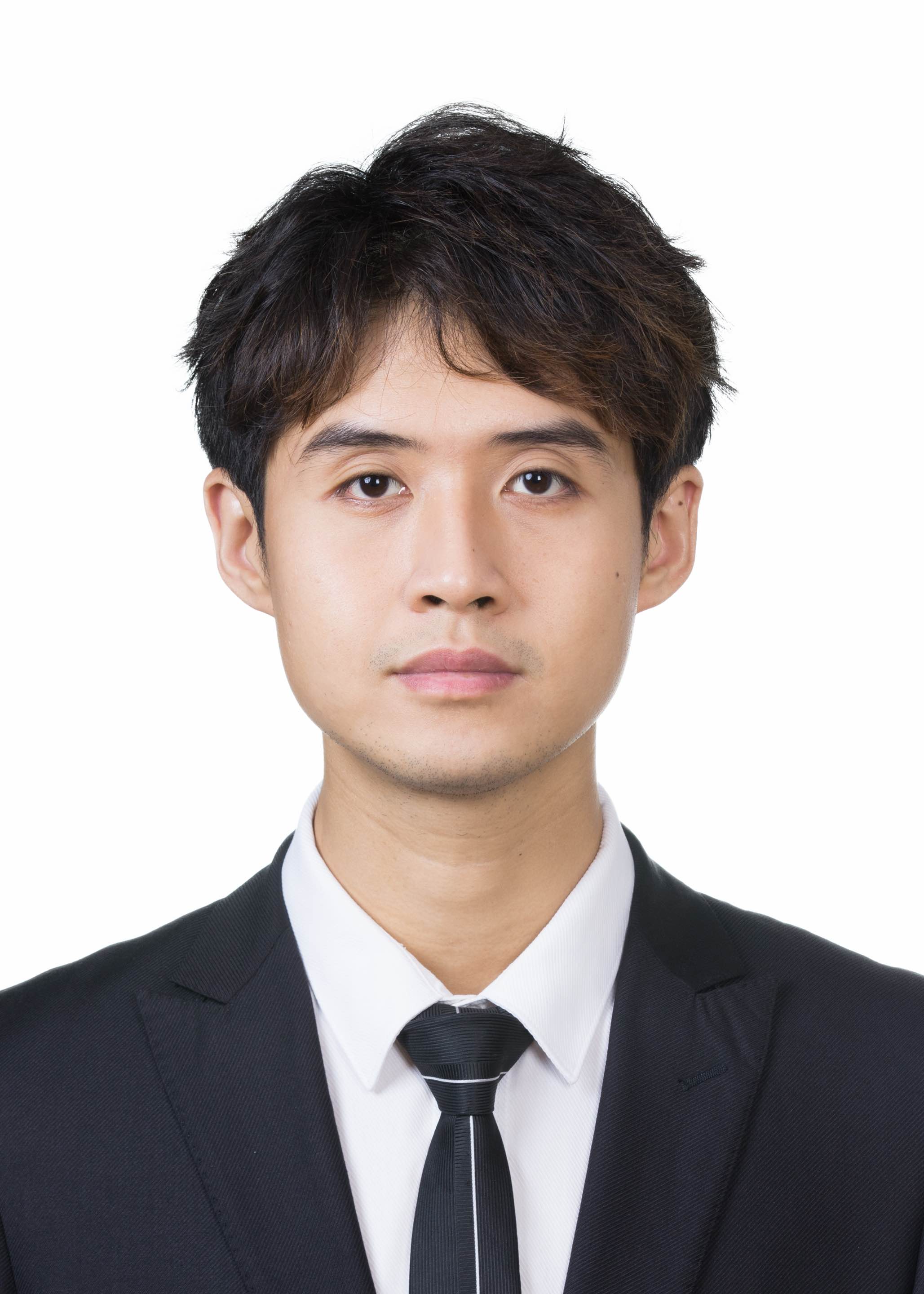}}]{Xunkai Li}
is currently pursuing the PhD degree in Beijing Institute of Technology, advised by Prof. Rong-Hua Li. He received the BS degree from Shandong University in 2022. His research interest lies in Data-centric Graph Intelligence (Data-centric AI, Graph Machine Learning, and AI4Science). He has published 10+ papers in top ML/DB/DM/AI conferences such as ICML, VLDB, WWW, AAAI.
\end{IEEEbiography}
\vspace{-1.6cm}
\begin{IEEEbiography}[{\includegraphics[width=1in,height=1in,clip,keepaspectratio]{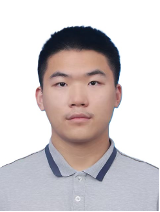}}]{Guohao Fu}
is an undergraduate student at Beijing Institute of Technology, majoring in Future Technology. His research experience spans Federated Multi-modal Graph Foundation Models and AI4S, with a focus on privacy-preserving distributed learning architectures and the application of deep learning paradigms to complex scientific discovery tasks.
\end{IEEEbiography}
\vspace{-1.6cm}
\begin{IEEEbiography}[{\includegraphics[width=1in,height=1in,clip,keepaspectratio]{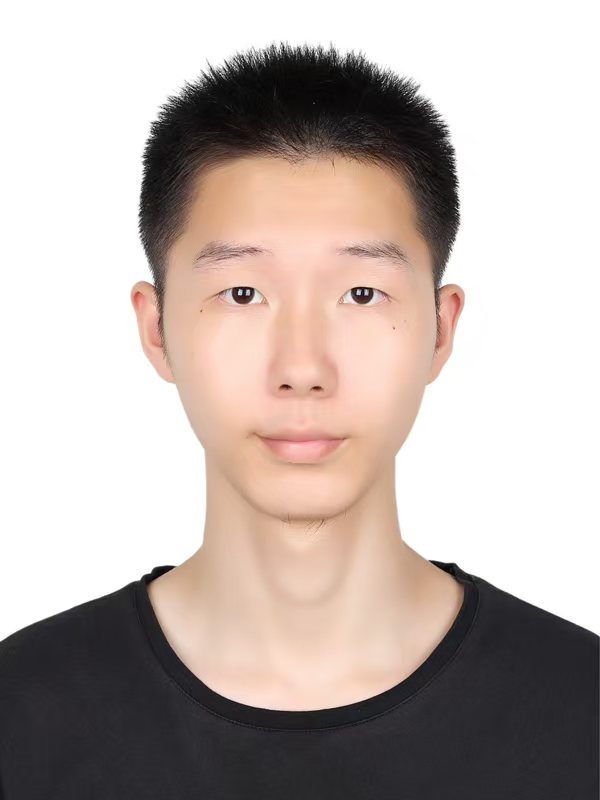}}]{Yuming Ai}
is currently pursuing the master degree majoring in Computer Science and Technology at Beijing Institute of Technology, with a research focus on multimodal graph learning and agentic learning. He has published papers in top conferences such as ICML, AAAI.
\end{IEEEbiography}
\vspace{-1.5cm}
\begin{IEEEbiography}[{\includegraphics[width=0.8in,height=1in,clip,keepaspectratio]{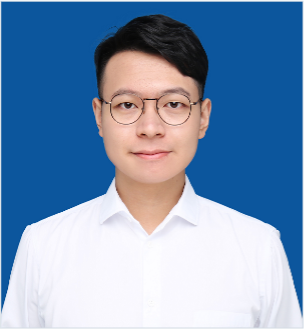}}]{Zhengyu Wu}
 is currently pursuing his PHD degree in Beijing Institute of Technology (BIT), Beijing, China. His research interests include Federated Graph Learning, directed graph learning, and social network analysis. He has co-authored papers published in top ML/DB/DM/AI conferences such as ICML, VLDB, WWW, AAAI.
\end{IEEEbiography}
\vspace{-1.5cm}
\begin{IEEEbiography}[{\includegraphics[width=1in,height=1in,clip,keepaspectratio]{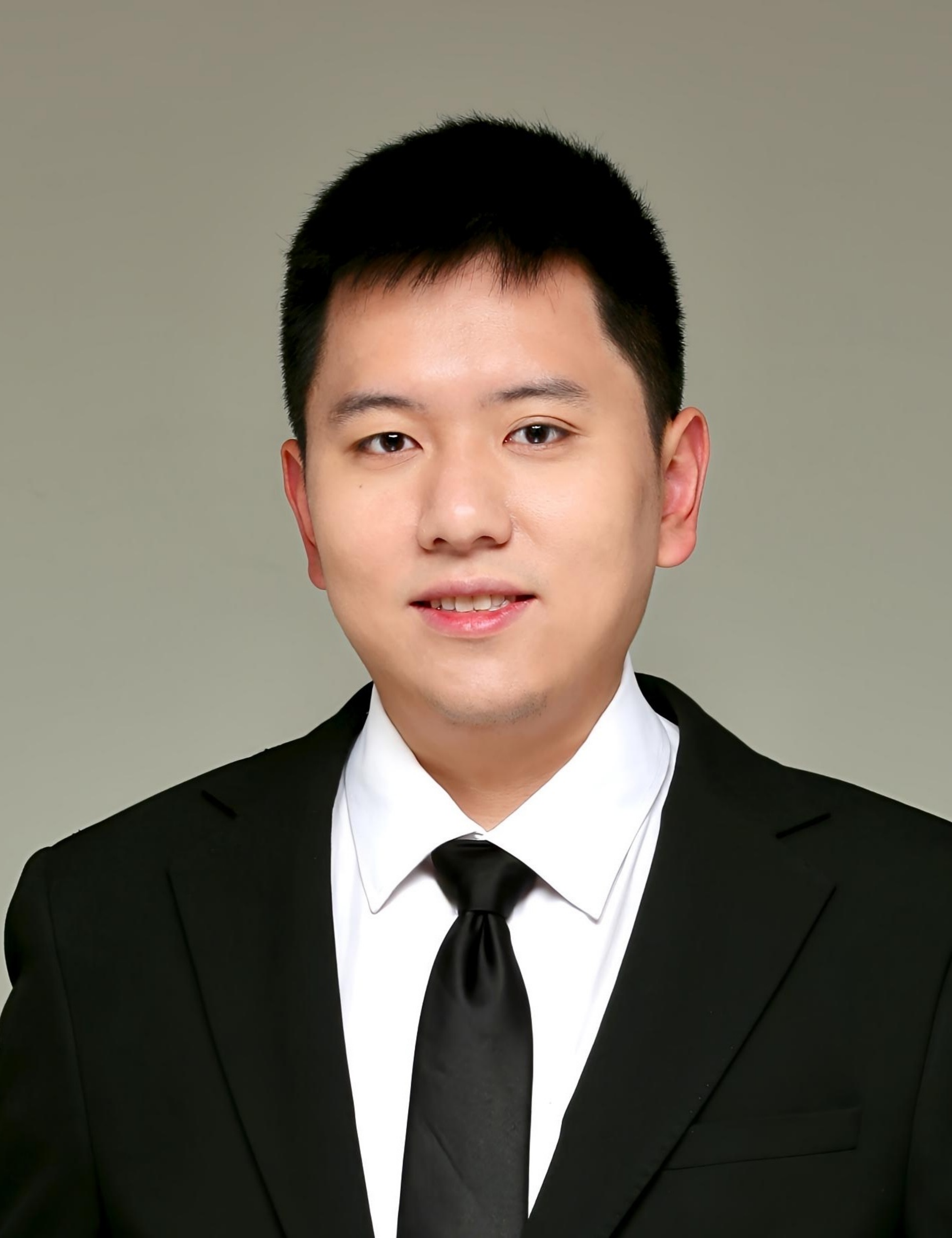}}]{Hongchao Qin}
 is currently an Assistant Professor with Beijing Institute of Technology, China. He received the B.S. degree in mathematics, M.E. degree and Ph.D. degree in computer science from Northeastern University, China. His research interests include community detection, graph mining, financial fraud detection, and RAG-based intelligent question answering systems.
\end{IEEEbiography}
\vspace{-1.5cm}
\begin{IEEEbiography}[{\includegraphics[width=1in,height=1in,clip,keepaspectratio]{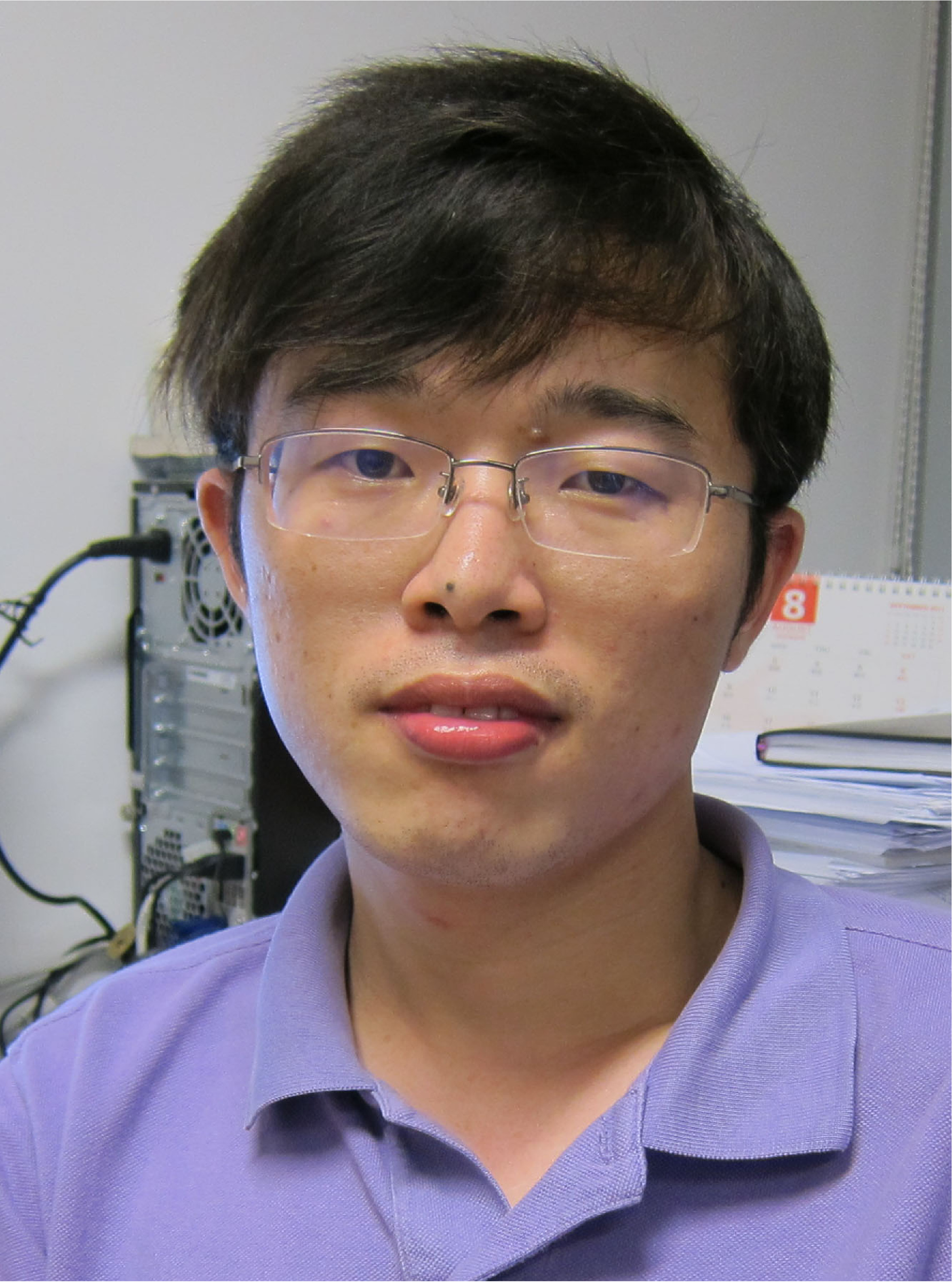}}]{Rong-Hua Li}
 received the PhD degree from the Chinese University of Hong Kong, in 2013. He is currently a professor with the Beijing Institute of Technology (BIT), Beijing, China. His research interests include graph data management and mining, social network analysis, graph computational systems, and graph-based machine learning.
\end{IEEEbiography}
\vspace{-1.5cm}
\begin{IEEEbiography}[{\includegraphics[width=1in,height=1in,clip,keepaspectratio]{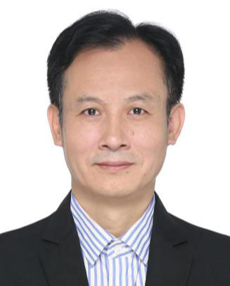}}]{Guoren Wang}
 received the BS, MS, and PhD degrees from the Department of Computer Science, Northeastern University, China, in 1988, 1991, and 1996, respectively. Currently, he is a professor with the Beijing Institute of Technology (BIT), Beijing, China. His research interests include graph data management, graph mining, and graph computational systems.
\end{IEEEbiography}

\vfill

\end{document}

%% file: sections/introduction.tex
\section{Introduction}
\label{sec:introduction}
\IEEEPARstart{G}{raph-structured} data in real-world applications are increasingly associated with heterogeneous multimodal attributes. For instance, a product node in an e-commerce graph may carry visual appearance alongside textual specifications, while a user node in a social platform contains both a profile image and a biographical description. Such multimodal-attributed graphs (MAGs), where nodes carry paired image and text features from frozen encoders such as CLIP~\cite{radford2021clip} alongside topological structure, offer substantially richer semantic signals than single-modality graphs, enabling tasks such as cross-modal retrieval and multimodal reasoning.

Graph foundation models (GFMs)~\cite{liu2023ofa,he2025unigraph,xia2024anygraph} have shown promise for learning transferable representations through pre-training followed by adaptation. However, most existing Multimodal GFMs assume centralized settings, lacking mechanisms for cross-modal interaction. For example, in healthcare and finance, MAG data are distributed across organizations and cannot be centralized due to regulations~\cite{kairouz2021advances,huang2024flsurvey}. Federated graph learning (FGL)~\cite{zhang2021fedsage,baek2023fedpub} enables collaborative training without sharing raw data, but current FGL methods and federated GFMs~\cite{zhu2025fedgfm,wu2025fedbook} remain limited to single-modality attributes. Combining the two paradigms is far from straightforward. Existing methods provide no mechanism to fuse multimodal node content with graph structure, nor to reconcile the resulting representations across clients whose data distributions differ. Thus, a clear gap remains at the intersection of multimodal graph learning and federated learning.

Designing a federated multimodal graph foundation model raises three fundamental challenges. \textbf{(C1) Privacy-constrained cross-modal alignment:} image and text modalities exhibit distinct feature distributions, so alignment must preserve cross-modal relational consistency while operating locally under privacy constraints. \textbf{(C2) Semantic-structural grounding:} MAGs combine modality-driven semantics and topology-driven structure, yet existing objectives often over-emphasize one at the expense of the other. \textbf{(C3) Heterogeneity-aware aggregation:} federated clients differ in graph scale, modality quality, and semantic distribution, yet the server cannot inspect raw data to estimate cross-client similarity. A privacy-preserving data summary is essential for channel-aware aggregation and personalization.

To illustrate these challenges, Fig.~\ref{fig:empirical} contrasts naive and principled federated multimodal GFMs (\S\ref{sec:empirical}), revealing three limitations. First, concatenating image and text features collapses modality-specific signals into near-identical distributions. Disentangling shared cross-modal and modality-private streams instead preserves the complementary diversity downstream tasks rely on. Second, jointly encoding semantics and structure allows semantics to override topology, linking semantically similar but structurally distant nodes across communities. Decoupling them into separate views before fusion prevents such links. Third, clients similar in feature-channel statistics often differ in graph-channel statistics, so a single similarity weight misallocates updates across channels. These observations yield three principles motivating our architecture.

Guided by these principles, we propose \textbf{FedGAMMA}, a \textbf{Fed}erated \textbf{G}raph \textbf{a}nd \textbf{M}ultimodal \textbf{A}lignment framework that treats federated multimodal graph foundation learning as a two-stage semantic-structural alignment problem. For cross-modal alignment (C1), FedGAMMA employs a shared-private semantic enhancer with symmetric cross-attention and OT-based distribution alignment~\cite{cuturi2013sinkhorn,courty2017otda}. For semantic-structural grounding (C2), a topology-aware graph fusion module decouples semantic and structural representations via semantic residual graphs and dual positional encodings. For heterogeneity-aware aggregation (C3), a dual-channel affinity mechanism summarizes clients into feature and graph centroids and performs channel-wise personalized aggregation without exposing raw data. FedGAMMA also introduces prompt-based fine-tuning with a shared prompt pool and channel-wise synchronization for lightweight downstream adaptation.

\textbf{Contributions.} \textbf{(1) New perspective.} We formulate federated multimodal graph foundation learning, unifying multimodal graph representation, privacy-preserving federation, and transferable pre-training in a single framework. \textbf{(2) New framework.} We propose FedGAMMA, a two-stage federated multimodal graph foundation model integrating shared-private cross-modal enhancement, topology-aware graph fusion, dual-channel affinity-aware aggregation, and prompt-based fine-tuning. \textbf{(3) SOTA performance.} Extensive experiments confirm that FedGAMMA outperforms all competitive baselines on $12$ multimodal graph datasets, with gains of up to $12.96\%$ over the best competing method. Beyond accuracy, it matches the strongest performance with $2$-$5\times$ fewer communication rounds, underscoring both its effectiveness and its efficiency.

%% file: sections/preliminaries.tex
\section{Preliminaries}
\label{sec:preliminaries}

\subsection{Multimodal-Attributed Graph Learning}

A multimodal-attributed graph is defined as $\mathcal{G}=(\mathcal{V},\mathcal{E},\mathcal{X}^{I},\mathcal{X}^{T})$, where $\mathcal{V}$ is the node set with $|\mathcal{V}|=N$, $\mathcal{E}\subseteq\mathcal{V}\times\mathcal{V}$ is the edge set, and each node $v_i\in\mathcal{V}$ is associated with an image feature vector $\mathbf{x}_i^{I}\in\mathbb{R}^{d}$ and a text feature vector $\mathbf{x}_i^{T}\in\mathbb{R}^{d}$. The feature matrices $\mathcal{X}^{I},\mathcal{X}^{T}\in\mathbb{R}^{N\times d}$ are obtained by passing raw modalities through frozen pretrained encoders such as CLIP~\cite{radford2021clip} with output dimension $d=768$. The adjacency matrix is denoted $\mathbf{A}\in\{0,1\}^{N\times N}$, where $\mathbf{A}_{ij}=1$ if $(v_i, v_j)\in\mathcal{E}$ and $\mathbf{A}_{ij}=0$ otherwise.

Multimodal graph learning~\cite{ngiam2011multimodal,wei2019mmgcn} aims to learn a node representation function $f:\mathcal{V}\to\mathbb{R}^{d'}$ that jointly encodes information from multiple modalities and graph topology through GNNs~\cite{kipf2017gcn}. Formally, given $\mathcal{G}$, the goal is to learn an encoder $f_\Theta$ such that the output representation matrix $\mathbf{Z}=f_\Theta(\mathcal{X}^{I},\mathcal{X}^{T},\mathbf{A})\in\mathbb{R}^{N\times d'}$ captures both cross-modal semantic relationships and topological structure. The learned representations support diverse downstream tasks: node classification $\hat{y}_i = g(\mathbf{z}_i)$, link prediction $\hat{e}_{ij}=\sigma(\mathbf{z}_i^\top\mathbf{z}_j)$, and cross-modal retrieval $\text{sim}(\mathbf{z}_i^{I},\mathbf{z}_j^{T})$, where $g$ and $\sigma$ denote a classifier head and the sigmoid function, respectively.

\subsection{Federated Multimodal Graph Foundation Models}

Consider $K$ clients, where client $k$ holds a local graph $\mathcal{G}_k=(\mathcal{V}_k,\mathcal{E}_k,\mathcal{X}_k^{I},\mathcal{X}_k^{T})$ drawn from distribution $\mathcal{D}_k$, exhibiting compound heterogeneity across graph scale, topology, modality quality, and semantic distribution~\cite{kairouz2021advances}. Raw data are never shared. The goal is to collaboratively train a transferable encoder $f_\Theta$ through a two-stage pipeline of federated pre-training~\cite{mcmahan2017fedavg} and prompt-based fine-tuning~\cite{lester2021prompt,zhou2022coop}.

\noindent\textbf{Federated pre-training.} Each client locally optimizes a self-supervised objective~\cite{gui2024ssl}, then uploads its parameters $\Theta_k$ and compact message $\mathbf{C}_k^{f},\mathbf{C}_k^{g}$ to the server. The server computes seperated weights and performs weighted aggregation:
\begin{equation}
\Theta^{g} \leftarrow \sum_{k=1}^{K} \frac{N_k}{\sum_{j}N_j} \cdot \mathbf{W}_k \odot \Theta_k,
\end{equation}
where $\mathbf{W}_k$ denotes the weight and $N_k$ is the number of local instances. The pre-training objective requires that $f_\Theta$ simultaneously satisfy the following: \textit{\ding{192}} model cross-modal interaction between $\mathcal{X}^{I}$ and $\mathcal{X}^{T}$ to capture complementary information across modalities; \textit{\ding{193}} preserve both semantic and structural information from $\mathbf{A}$, ensuring that topology-driven patterns are not overwritten by modality-driven features; and \textit{\ding{194}} account for cross-client heterogeneity during aggregation, avoiding negative transfer from uniform averaging over diverse client distributions~\cite{liu2024fola}.

\noindent\textbf{Fine-tuning.} After pre-training, the global encoder $\Theta^{g}$ is frozen and distributed to all clients. Downstream adaptation proceeds through graph-to-prompt generation, prompt-pool selection, and channel-wise prompt synchronization, so that only lightweight parameters $\phi$ with $|\phi|\ll|\Theta^{g}|$ need to be updated and communicated.
With the backbone frozen, client $k$ minimizes only over prompt-module parameters $\phi$ and task-head parameters $\psi$:
\begin{equation}
\label{eq:prelim-ftobj}
\min_{\phi,\psi}\;\mathcal{L}_{\text{task}}\bigl(\mathcal{G}_k,\mathcal{Y}_k;\Theta^{g},\phi,\psi\bigr),
\end{equation}
where $|\phi|+|\psi|\ll|\Theta^{g}|$. After local training, each client uploads its task-head and prompt-module parameters together with compact prompt centroids. The server applies the same dual-channel affinity rule used during pre-training to aggregate these parameters channel-wise, preserving channel-specific personalization while enabling federated knowledge transfer in the prompt space.

%% file: sections/related_works.tex
\section{Related Work}
\label{sec:related_work}

\subsection{Federated Graph Foundation Models}

FedGFM~\cite{zhu2025fedgfm} decouples structural and semantic components inside the GNN encoder and applies contrastive graph-text alignment, making it the first model to combine graph foundation pre-training with federated learning. FedBook~\cite{wu2025fedbook} projects aligned representations onto a shared discrete codebook via vector quantization, enabling intra-domain and inter-domain knowledge transfer across clients. FedGALA~\cite{li2025fedgala} instead favors continuous latent alignment between GNNs and frozen language models, followed by local prompt-based fine-tuning for downstream adaptation. Their multimodal extensions simply fuse image and text features before feeding them into the original unimodal pipelines.

These methods are inherently unimodal, and their multimodal extensions discard modality-specific information through feature-level fusion; FedGFM and FedBook further incur quantization error from discrete codebook projection. None jointly address cross-modal fusion, semantic-structural grounding, and channel-aware aggregation. Beyond graph data, cross-modal federated learning~\cite{yang2024cmfhar} disentangles modality-agnostic and modality-specific representations across clients, and federated feature alignment~\cite{zhou2024fedfa} mitigates cross-client feature shift, but neither models graph topology. The broader federated graph learning methods~\cite{li2024fedgta,huang2024fgssl,zhu2024fedtad,xie2021gcfl,tan2022fedproto,li2024adafgl,zhang2021fedsage,wang2024fedscope,luo2026fairfgl,ye2024fedcausal} and graph self-supervised pre-training methods~\cite{wang2024gvtmae,kipf2016variational,velickovic2019dgi,gui2024ssl} address related but distinct challenges under unimodal or centralized settings.

\subsection{Multimodal Graph Learning Models}

\noindent\textbf{Multimodal Graph Neural Networks.}
Recent work on multimodal graph learning has moved beyond simple concatenation toward designs that explicitly model cross-modal relationships~\cite{xu2023mmtsurvey} and structural nuances.
CAMPA~\cite{wan2026campa} resolves modal conflict in decoupled multimodal GNNs through cross-modal aligned propagation.
TMTE~\cite{zhu2026tmte} co-evolves graph structure and multimodal representations under task-aware metric learning.
RoleMAG~\cite{zuo2026rolemag} learns edge-level neighbor roles and routes neighborhood information through role-specific propagation channels.
NTSFormer~\cite{hu2026ntsformer} targets isolated cold-start node classification with a self-teaching Graph Transformer, and MIG-GT~\cite{hu2025miggt} employs modality-independent receptive fields with separate GNNs per modality and a sampling-based global transformer for multimodal recommendation.

These methods advance multimodal graph learning through decoupled propagation, role-aware routing, or adaptive topology learning. However, they are designed for centralized settings and do not address privacy constraints or cross-client heterogeneity. Their federated counterparts, obtained via FedAvg~\cite{mcmahan2017fedavg}, lack server-side mechanisms for aligning inter-client representations and cannot distinguish feature-level and graph-level similarity patterns.

\noindent\textbf{Multimodal Graph Foundation Models.}
Mario~\cite{sun2026mario} introduces a graph-conditioned vision-language model that injects graph topology into dual-tower encoders via cross-node attention, paired with a modality-adaptive prompt router that selects the optimal modality view per node for LLM-based reasoning. PLANET~\cite{li2025planet} decouples modality interaction and alignment through a divide-and-conquer strategy: EDG enriches embeddings with topology-aware cross-modal context at the embedding level, while NDR anchors heterogeneous signals into a discretized semantic space for global alignment at the node level. UniGraph2~\cite{he2025unigraph2} extends UniGraph to multimodal graphs by learning a unified embedding space via masked autoencoding over multimodal features and contrastive graph-structure objectives. Applying FedAvg yields the federated baselines Fed-Mario, Fed-PLANET, and Fed-UniGraph2.

These centralized MGFMs are designed for full graph corpora. When federated, local-only alignment produces inconsistent semantic geometry across clients, and no server-side similarity mechanism exists.

%% file: sections/empirical_investigation.tex
\section{Empirical Investigation}
\label{sec:empirical}

Federated multimodal graph learning operates at two complementary levels. At the \textbf{modality level}, each client independently performs multimodal graph encoding: a local encoder must fuse heterogeneous image and text features while preserving both cross-modal complementarity and topological structure. At the \textbf{server level}, the server must aggregate model updates from clients that differ in graph scale, modality quality, and semantic distribution, without access to raw data. Naively combining these two levels without targeted design leads to three systematic problems, namely modality collapse, structural drift, and cross-channel aggregation inconsistency, the first two arising at the modality level and the last at the server level. In this section, we characterize each problem through a controlled empirical comparison between a naive federated GFM baseline (\textbf{FedGFM}) and a principled design that incorporates modality-aware encoding and channel-wise aggregation (\textbf{FedMGFM}), and show that each problem directly motivates one of the three core modules in the proposed FedGAMMA framework (\S\ref{sec:method}).

\begin{figure*}[t]
\centering
\includegraphics[width=\linewidth]{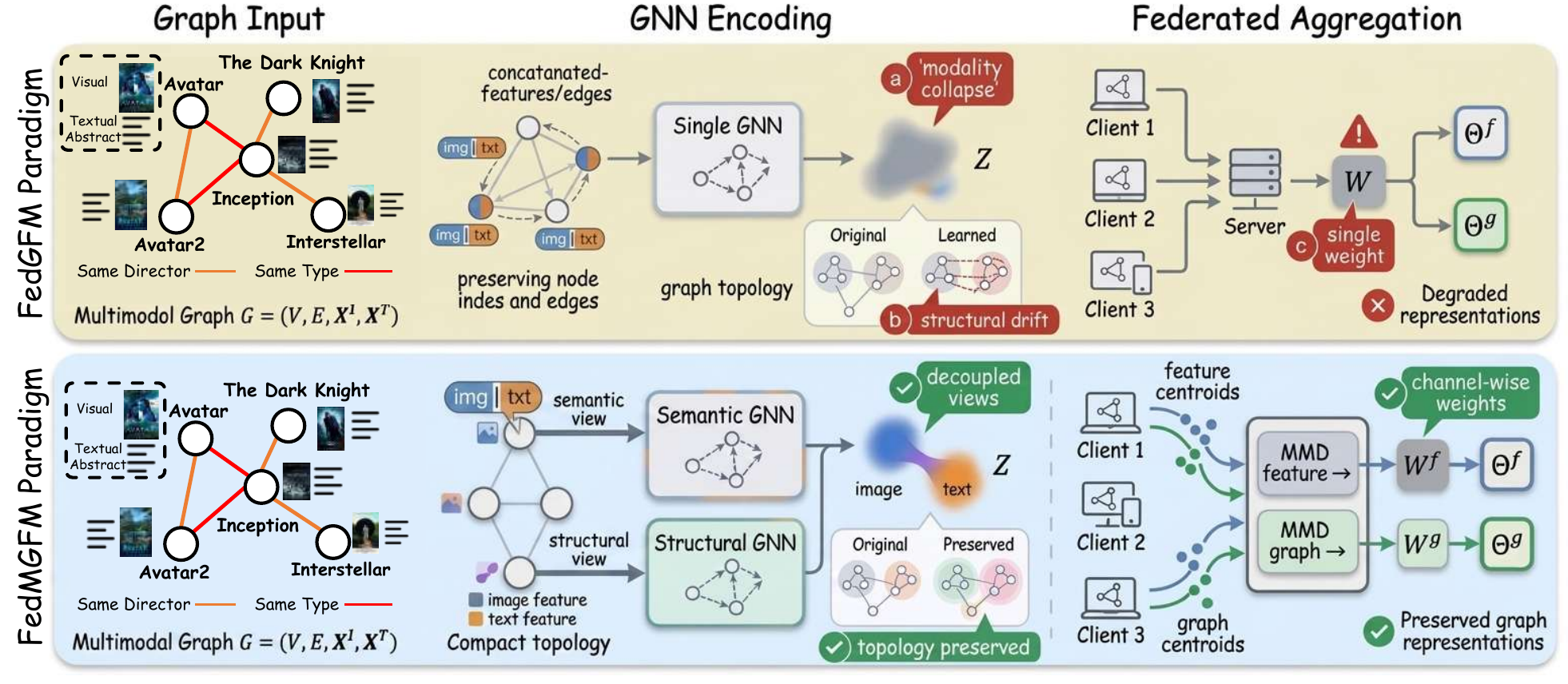}
\caption{\textbf{Comparison of FedGFM and FedMGFM for federated multi-modal graph learning.}
FedGFM suffers from modality collapse and misallocation due to single-weight aggregation, while FedMGFM adopts decoupled multi-view modeling and channel-wise aggregation to preserve topology.}
\label{fig:empirical}
\end{figure*}

\subsection{Modality Collapse}
\label{sec:modality-collapse}

\noindent\textbf{Observation.} In the FedGFM baseline, image and text features are concatenated and passed through a shared GNN encoder~\cite{kipf2017gcn}. As shown in Fig.~\ref{fig:empirical}, this naive fusion compresses heterogeneous modality information into a single representation, so the image- and text-derived representations collapse into a near-identical distribution. In contrast, FedMGFM separates a shared cross-modal stream from modality-private streams: symmetric cross-attention extracts multimodal consensus while private encoders preserve modality-specific complementary signals, maintaining cross-modal alignment.

\noindent\textbf{Claims.} Shallow multimodal fusion erases modality-specific information, and the complementary cues essential for cross-modal tasks are lost during encoding rather than at the task head, so no downstream supervision can later recover them. The collapse is avoided only when the encoder disentangles shared and private streams, which indicates that shared-only architectures are fundamentally insufficient for multimodal graph representation.

\noindent\textbf{Key Insight \ding{182}} \textit{Shallow multimodal fusion erases modality-specific information. Effective cross-modal encoding requires explicit shared-private separation to preserve complementary signals from each modality.}

\subsection{Structural Drift}
\label{sec:structural-drift}

\noindent\textbf{Observation.} In FedGFM, the encoder processes fused multimodal node features through a conventional GNN without distinguishing between semantic and structural signals. As illustrated in Fig.~\ref{fig:empirical}, this entangled encoding causes the learned representations to be driven by feature similarity rather than graph topology: semantically similar but structurally distant nodes are connected across community boundaries, distorting the original graph structure. As a result, the learned graph $\tilde{G}$ induced by these representations departs from the original graph $G$, confirming that feature-driven proximity overrides topological structure. FedMGFM addresses this by decoupling the encoding into complementary views: a semantic view that propagates over the original graph together with a similarity graph built from shared features, and a structural-role view that propagates over a diffusion graph with dual positional encodings. The two views are fused only after independent formation.

\noindent\textbf{Claims.} Semantic content overwrites the structural prior, a phenomenon we term \textit{structural drift}. Entangled semantic-structural encoding causes the encoder to rely on feature similarity, distorting the topological signal that downstream structure-dependent tasks require. The two signals therefore compete for the same representation capacity, and must be propagated through decoupled views so that neither suppresses the other.

\noindent\textbf{Key Insight \ding{183}} \textit{Entangled encoding allows semantic features to overwrite topological structure. Decoupled semantic and structural views, fused only after independent formation, prevent structural drift and preserve graph-level fidelity.}

\subsection{Cross-Channel Inconsistency}
\label{sec:cross-channel}

\noindent\textbf{Observation.} In FedGFM, the server aggregates model updates using a single similarity weight shared across multimodal-fusion parameters $\Theta_{f}$ and graph-fusion parameters $\Theta_{g}$. As shown in Fig.~\ref{fig:empirical}, this single-similarity scheme ignores the multi-dimensional nature of client heterogeneity. Each client is summarized by a feature-channel centroid and a graph-channel centroid, and these two summaries often disagree, so clients close in the feature centroid space are frequently distant in the graph centroid space. FedMGFM resolves this through dual-channel affinity-aware aggregation: separate MMD-based affinity matrices are computed for feature-channel and graph-channel parameters, each parameter group is routed through its own similarity structure, and personalized interpolation is performed per client.

\noindent\textbf{Claims.} Feature-channel and graph-channel client similarities are largely independent, since modality distribution and graph structure vary along separate axes. A single-similarity aggregation scheme is therefore inevitably suboptimal, because the weight best suited to one channel misallocates updates for the other. Effective federated aggregation must instead be channel-wise, routing each parameter group through its own similarity structure.

\noindent\textbf{Key Insight \ding{184}} \textit{Client similarities are largely independent in different channels. Single-similarity aggregation inevitably misallocates weights for at least one parameter group. Channel-wise affinity-aware aggregation resolves this by routing each parameter group through its own similarity structure.}

\vspace{0.5em}
\noindent\textbf{Remarks.} The comparison (Fig.~\ref{fig:empirical}) reveals a consistent pattern that naive combinations of multimodal encoding and federated averaging introduce representation collapse, structural drift, and aggregation misallocation. These are not edge cases but systematic consequences of ignoring the multi-dimensional nature of multimodal graph data in federated settings. The three design principles demonstrated by FedMGFM, i.e., shared-private cross-modal separation, decoupled semantic-structural views, and channel-wise affinity-aware aggregation, are instantiated as concrete modules in FedGAMMA (\S\ref{sec:method}).

%% file: sections/method.tex
\section{Method}
\label{sec:method}

\begin{figure*}[t]
\centering
\includegraphics[width=\linewidth]{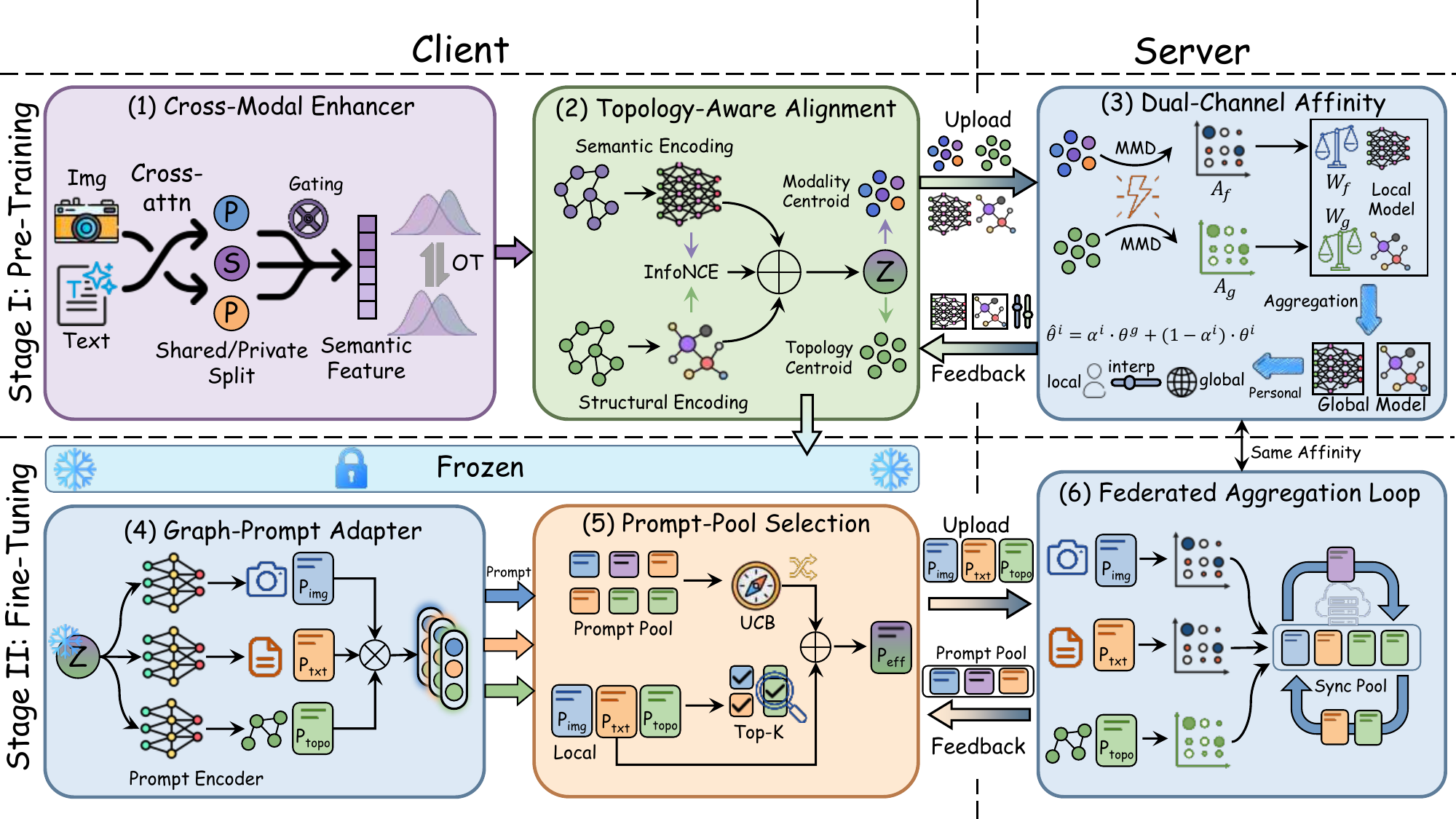}
\caption{\textbf{Overview of the FedGAMMA framework.} The framework follows a two-stage pipeline. \textbf{Pre-training} (left): the Cross-Modal Semantic Enhancer disentangles shared and modality-private streams, and the Topology-Aware Graph Fusion module propagates features over complementary semantic and structural views. \textbf{Federated synchronization} (center): clients upload feature and graph centroids; the server computes dual-channel affinity matrices for channel-wise aggregation and personalized interpolation. \textbf{Prompt-based fine-tuning} (right): a Graph-to-Prompt Adapter generates modality-specific prompts selected from a global Prompt Pool, with channel-aware synchronization across clients.}
\label{fig:framework}
\end{figure*}

\subsection{Overview}

FedGAMMA is a federated multimodal graph foundation framework built on a two-stage pipeline: \emph{federated pre-training} followed by \emph{prompt-based fine-tuning}~\cite{lester2021prompt,zhou2022coop}. The pre-training stage learns a transferable encoder that decouples semantic and structural signals across heterogeneous clients through contrastive learning~\cite{oord2018representation,zhu2020grace} and optimal-transport-based alignment~\cite{cuturi2013sinkhorn}, and the fine-tuning stage adapts that encoder through lightweight graph-aware prompts and channel-wise federated synchronization.

As shown in Fig.~\ref{fig:framework}, the pre-training stage uses a \emph{Cross-Modal Semantic Enhancer} for shared-private disentanglement via symmetric cross-attention over CLIP features~\cite{radford2021clip}, a \emph{Topology-Aware Graph Fusion} module with a GNN backbone~\cite{kipf2017gcn,velickovic2018gat} to decouple semantic propagation and structural-role modeling through dual positional encodings and multi-hop diffusion, and a \emph{Dual-Channel Affinity Aggregator} based on MMD~\cite{gretton2012mmd} for heterogeneity-aware synchronization. The fine-tuning stage uses a \emph{Graph-to-Prompt Adapter} to inject graph-aware prompts per modality, a \emph{Prompt-Pool Selection} mechanism with UCB-style exploration, and a \emph{Federated Aggregation Loop} that synchronizes prompts along the same dual-channel logic.

\subsection{Pre-training: Cross-Modal Semantic Enhancer}
\label{subsec:enhancer}

This module disentangles shared cross-modal signals from modality-specific detail before graph propagation, so the downstream stage receives a stable, decoupled representation.

\noindent\textbf{Motivation.}\quad
Image and text in the same node share a coarse cross-modal core, yet each retains irreducible modality-specific information. Under non-IID clients, single-stream fusion is dominated by whichever modality carries more local gradient energy, drifting the shared signal toward one side until modality collapse emerges. The remedy is to decompose the input into one shared and two modality-specific subspaces and recover each explicitly before recomposing them on demand.

\noindent\textbf{Architecture.}\quad
Given paired modality features $\mathcal{X}^{I},\mathcal{X}^{T}\in\mathbb{R}^{N\times d}$, FedGAMMA produces a shared and a private stream per modality. The shared streams $\mathbf{S}^{I},\mathbf{S}^{T}$ come from $L_{f}$ stacked \emph{symmetric cross-attention} blocks in which each modality attends to the other, exchanging cross-modal information. The private stream $\mathbf{P}^{m}$ comes from a modality-specific MLP and retains the detail outside the cross-attention range. A learned input-dependent gate then recomposes the two streams into the final feature $\mathbf{F}^{m}$, adding the raw input as a residual so the recomposition stays lossless. Averaging the two shared streams gives the cross-modal anchor $\bar{\mathbf{S}}$ exported to the next module, and a reconstruction head maps each $\mathbf{F}^{m}$ back to $\hat{\mathcal{X}}^{m}$ for the regularizer below.

\noindent\textbf{Training objective.}\quad
The module is trained by a Sinkhorn--Wasserstein alignment loss $\mathcal{L}_{\text{align}}$ that applies an optimal-transport distance between the two shared streams $\mathbf{S}^{I}$ and $\mathbf{S}^{T}$. It is regularized by a single composite term $\mathcal{L}_{\text{dec}}^{\text{m}}$ that sums, over both modalities, a disentanglement penalty on the absolute cosine similarity between the shared and private streams and a reconstruction error between each recovered input $\hat{\mathcal{X}}^{m}$ and its original feature. The combined module objective is
\begin{equation}
\label{eq:enhancer-loss}
\mathcal{L}^{\text{enh}} \;=\; \mathcal{L}_{\text{align}} \;+\; \theta\,\mathcal{L}_{\text{dec}}^{\text{m}},
\end{equation}
with $\theta$ the only regularizer coefficient.

\subsection{Pre-training: Topology-Aware Graph Fusion}
\label{subsec:graphfusion}

This module produces the node representation by combining complementary semantic and structural-role views, so that semantic propagation and structural-role modeling remain separable instead of being entangled inside a single GNN.

\noindent\textbf{Motivation.}\quad
Federated graphs vary in two largely independent ways: what the nodes are about and how they connect. A single GNN entangles both, and the semantic gradient overwhelms the structural one, eroding structural roles over training. Decoupling semantic propagation from structural-role modeling and recombining the views \emph{post hoc} fixes this drift and yields a clean structural-role embedding, which also serves as the clustering signature for the next module.

\noindent\textbf{Architecture.}\quad
Let $\mathcal{X}^{n}=[\mathbf{F}^{I}\Vert\mathbf{F}^{T}]\!\in\!\mathbb{R}^{N\times 2d}$ be the concatenated final modality features from the enhancer. From this input FedGAMMA derives two complementary views, a semantic view and a structural-role view.

The \emph{semantic view} captures what the nodes are about by propagating the fused modality features through a GNN. Because a client's observed topology is often sparse, FedGAMMA augments it with a semantic residual graph $\mathcal{G}_{\text{res}}$ built from the shared anchor, which links each node to its $k$ nearest neighbors in $\bar{\mathbf{S}}$ under cosine similarity after discarding edges already present in $\mathcal{G}$, thereby supplying semantically proximal yet topologically absent links. Propagating over the observed graph and over $\mathcal{G}_{\text{res}}$ and combining the two yields the semantic embedding $\mathbf{H}_{\text{sem}}$.

The \emph{structural-role view} instead models how the nodes connect, independently of modality content. It feeds a purely structural node descriptor $\mathcal{X}^{\!\text{str}}$, which combines random-walk and Laplacian positional encodings with simple local statistics such as degree and two-hop count, into a GNN over a multi-hop diffusion graph $\mathcal{G}_{\text{diff}}$ that aggregates powers of the normalized adjacency with geometrically decaying weights, yielding the structural-role embedding $\mathbf{H}_{s}$.

During pre-training, the two views are combined into the node representation $\mathbf{Z}$ by a parameter-free average. At fine-tuning, a lightweight learnable router reweights them per node through a one-layer MLP, allowing downstream tasks to emphasize whichever view best suits their inductive demands.

\noindent\textbf{Training objective.}\quad
The module is trained by an InfoNCE contrastive loss $\mathcal{L}_{\text{nce}}$ that aligns the structural-role view $\mathbf{H}_{s}$ with the fused representation $\mathbf{Z}$. It is regularized by a single composite term $\mathcal{L}_{\text{reg}}^{\text{g}}$ that sums three stability constraints: a consistency term keeping the observed-graph and residual-graph components of the semantic view aligned, a decoupling term driving the structural-role view $\mathbf{H}_{s}$ away from the semantic view $\mathbf{H}_{\text{sem}}$, and a self-supervised link-prediction loss on positive and sampled negative edges. The composite graph-level objective combines the contrastive loss with the regularizer under their respective weights,
\begin{equation}
\label{eq:graph-loss}
\mathcal{L}^{\text{g}} \;=\; \beta\,\mathcal{L}_{\text{nce}} \;+\; \lambda_{\text{reg}}\,\mathcal{L}_{\text{reg}}^{\text{g}},
\end{equation}
where $\beta$ is the principal contrastive weight and $\lambda_{\text{reg}}$ is the only coefficient governing the regularizer.

\providecommand{\algfontsize}{\normalsize}

\begin{algorithm}[t]
\caption{Federated Pre-training of FedGAMMA}
\label{alg:pretrain}
\algfontsize
\begin{tabbing}
xx\=xx\=xx\=xx\=\kill
\textbf{Input:} $K$ clients, local graphs $\{\mathcal{G}_{k}\}$; rounds $T$; epochs $E$\\
\textbf{Output:} pre-trained parameters $\Theta$\\
1: Server initializes $\Theta^{(0)}$\\
2: \textbf{for} $t=1,\dots,T$ \textbf{do}\\
3: \>Broadcast $\Theta^{(t-1)}$\\
4: \>\textbf{for each} client $k$ \textbf{in parallel do}\\
5: \>\>Compute $\mathbf{S}^{m},\mathbf{P}^{m}$ and gate $\mathbf{F}^{m}$; anchor $\bar{\mathbf{S}}$\\
6: \>\>Build $\mathcal{G}_{\text{res}},\mathcal{G}_{\text{diff}}$; propagate $\mathbf{H}_{\text{sem}},\mathbf{H}_{s}\!\to\!\mathbf{Z}$\\
7: \>\>\textbf{for} $e=1,\dots,E$ \textbf{do} minimize $\mathcal{L}^{\text{pre}}$ via \eqref{eq:pretrain-loss}\\
8: \>\>$C_{k}^{f}\!\leftarrow\!\text{KMeans}(\mathbf{Z},M)$, $C_{k}^{g}\!\leftarrow\!\text{KMeans}(\mathbf{H}_{s},M)$; add $\mathbf{e}_{k,m}^{c}$\\
9: \>\>Upload $\Theta_{k},C_{k}^{f},C_{k}^{g}$\\
10:\>\textbf{end for}\\
11:\>Server: $A^{f},A^{g}$ via \eqref{eq:affinity}; aggregate $\Theta^{(t)}$ via \eqref{eq:aggregation}\\
12:\>Clients personalize via \eqref{eq:personalize}\\
13:\textbf{end for}
\end{tabbing}
\end{algorithm}

\subsection{Pre-training: Dual-Channel Affinity Aggregation}
\label{subsec:aggregation}

This module synchronizes trained parameters across clients by routing multimodal-fusion and graph-fusion parameters through two distinct affinity matrices, ensuring each parameter group attends to neighbors aligned with its inductive role.

\noindent\textbf{Motivation.}\quad
After local pre-training, each client holds two preferences that need not align: a modality preference in how its fused embeddings cluster, and a topology preference in how its structural-role embeddings cluster. A single similarity score must therefore misroute one of the two parameter groups, so FedGAMMA computes one affinity per channel and aggregates each group with its own weights, which Theorem~\ref{thm:dual-channel} formalizes as a strict gap over any single-channel scheme.

\noindent\textbf{Architecture.}\quad
Algorithm~\ref{alg:pretrain} states the full federated loop, whose components and equations we now explain. After local training, each client compresses its node representations into $M$ centroids per channel by $K$-means, forming $C_{i}^{f}$ from the fused embedding $\mathbf{Z}$ and $C_{i}^{g}$ from the structural-role embedding $\mathbf{H}_{s}$. Each centroid pairs its cluster mean $\boldsymbol{\mu}_{i,m}^{c}$ with a \emph{spectral energy descriptor} $\mathbf{e}_{i,m}^{c}$ that encodes the cluster's internal geometry, taken as the $d_c$ smallest eigenvalues of the normalized Laplacian $\mathbf{L}_{i,m}$ of the cosine-similarity graph over its members, normalized by their sum,
\begin{equation}
\label{eq:spectral-energy}
\mathbf{e}_{i,m}^{c} \;=\; \frac{[\lambda_{1},\lambda_{2},\dots,\lambda_{d_c}]^{\top}}{\sum_{l=1}^{d_c}\lambda_{l}}, \qquad c\in\{f,g\}.
\end{equation}
These descriptors keep clusters that share a mean but differ in internal connectivity distinguishable, and $d_c$ defaults to the feature dimension $d$. The server compares clients \emph{per channel} through the MMD between their centroid sets,
\begin{equation}
\label{eq:affinity}
A_{ij}^{c} \;=\; \exp\!\Bigl(-\frac{\mathrm{MMD}(C_{i}^{c},C_{j}^{c})}{\eta}\Bigr), \quad c\in\{f,g\},
\end{equation}
where $\eta$ is a bandwidth, and normalizing the row sums of $A^{c}$ across clients gives the channel weight $w_{i}^{c}\!\propto\!\sum_{j}A_{ij}^{c}$. Each parameter group is then aggregated with the weights of its matching channel,
\begin{equation}
\label{eq:aggregation}
\Theta_{t+1}^{p} \;=\;
\begin{cases}
\sum_{i} w_{i}^{f}\,\Theta_{i,t}^{p}, & p\in\Theta_{f}, \\[2pt]
\sum_{i} w_{i}^{g}\,\Theta_{i,t}^{p}, & p\in\Theta_{g}, \\[2pt]
\sum_{i} \bar w_{i}\,\Theta_{i,t}^{p}, & p\in\Theta_{s},
\end{cases}
\end{equation}
where $\bar w_{i}\!=\!\tfrac{1}{2}(w_{i}^{f}\!+\!w_{i}^{g})$ and $\Theta_{f},\Theta_{g},\Theta_{s}$ are the multimodal-fusion, graph-fusion, and shared parameters. After broadcast, each client interpolates between the global and its local model per channel,
\begin{equation}
\label{eq:personalize}
\Theta^{\text{new},c}_{k} \;=\; a_{k}^{c}\,\Theta^{\text{global},c} \;+\; (1\!-\!a_{k}^{c})\,\Theta^{\text{local},c}_{k}, c\in\{f,g,s\},
\end{equation}
where $a_{k}^{c}$ is the affinity of Eq.~\eqref{eq:affinity} between client $k$'s centroids and the aggregated global centroids and $a_{k}^{s}$ averages the two channels, so clients typical of the population follow the global update while atypical clients keep more of their local optimum.

\noindent\textbf{Training objective.}\quad
The aggregator has no trainable parameters, so the client-side pre-training loss simply sums the two module objectives,
\begin{equation}
\label{eq:pretrain-loss}
\mathcal{L}^{\text{pre}} \;=\; \mathcal{L}^{\text{enh}} \;+\; \mathcal{L}^{\text{g}},
\end{equation}
The stage is thus driven by two principal losses, $\mathcal{L}_{\text{align}}$ and $\mathcal{L}_{\text{nce}}$, with the auxiliary constraints bundled in $\mathcal{L}_{\text{dec}}^{\text{m}}$ and $\mathcal{L}_{\text{reg}}^{\text{g}}$ under one coefficient each.

\subsection{Fine-tuning: Graph-to-Prompt Adapter}
\label{subsec:gtp}

This module converts the pretrained structural representation into channel-specific prompts that multiplicatively modulate the image, text, and topology streams, enabling backbone-frozen downstream adaptation across all channels.

\noindent\textbf{Motivation.}\quad
Full-parameter fine-tuning of a federated multimodal graph backbone is communication-expensive and overfits local idiosyncrasies. At the same time, a single static prompt shared across clients cannot reflect that different downstream tasks and different clients place their adaptation budget on different channels. To resolve both, FedGAMMA makes the prompt a function of the local graph and applies one multiplicatively gated prompt per channel, restoring per-channel selectivity while keeping the backbone frozen.

\noindent\textbf{Architecture.}\quad
The pretrained backbone is frozen, and only the prompt parameters and a small task head are updated and communicated. From the pretrained node representation $\mathbf{Z}$, three two-layer MLPs synthesize one prompt $P_{m}$ per channel for image, text, and topology. The image and text prompts multiplicatively gate their dynamic features $\mathcal{X}^{m}_{\text{dyn}}$, the inference-time recomposition of the raw, shared, and private streams, through an additive-sigmoid form,
\begin{equation}
\label{eq:prompt-mod}
\tilde{\mathcal{X}}^{m} \;=\; \mathcal{X}^{m}_{\text{dyn}} \odot \bigl(\mathbf{1}+\sigma(P_{m})\bigr), \qquad m\in\{I,T\},
\end{equation}
which bounds the modulation to $[1,2]$, so a prompt can amplify but never suppress a feature, keeping the gate from acting as a hard mask under federated drift. The topology prompt instead reweights the structural branch, scaling each node's structural input and every diffusion edge, so propagation can be tuned without rebuilding the diffusion graph.

\subsection{Fine-tuning: Prompt-Pool Selection}
\label{subsec:promptpool}

This module maintains a shared pool of candidate prompts and lets each client pick a small subset per step under an exploration--exploitation rule, so that the pool preserves diversity instead of collapsing onto a single dominant prompt under federated averaging.

\noindent\textbf{Motivation.}\quad
Different clients exercise different combinations of channels, so a single global prompt averages away the very diversity that motivates federation. A small pool with per-client selection restores it, provided the selection rule balances reward against exploration so the pool does not collapse after a few rounds.

\noindent\textbf{Architecture.}\quad
The server maintains a pool $\mathcal{P}$ of $S$ candidate prompts shared across all clients, so a prompt that proves useful on one client becomes reusable knowledge for the others. Each client $i$ keeps a private running reward $r_{i,k}$ and usage count $n_{i,k}$ for every slot $k$, where $r_{i,k}$ summarizes the downstream task performance that slot has yielded locally. At each step the client scores every slot by a UCB-style rule,
\begin{equation}
\label{eq:ucb}
\operatorname{score}_{i}(k) \;=\; r_{i,k} \;+\; \kappa\,\sqrt{\frac{\log(\sum_{j}\!n_{i,j}+2)}{n_{i,k}+1}},
\end{equation}
whose first term exploits high-reward slots while the second term, weighted by $\kappa$, adds an exploration bonus that keeps rarely used slots attractive. The smoothing constants in the logarithm and denominator keep this bonus well defined when a slot has not yet been pulled. The client then selects the top-$K_{\text{sel}}$ slots, averages them into a single pool prompt, and fuses it with the local graph-to-prompt prompt as a residual with a small mixing weight $\gamma$, so the effective prompt combines shared pool memory with local specialization. After the forward pass, it updates each selected slot's reward by an exponential moving average over the newly realized task performance and increments the count, and these statistics persist across rounds so the pool steadily accumulates cross-client adaptation experience.

\subsection{Fine-tuning: Federated Aggregation Loop}
\label{subsec:ftloop}

This module synchronizes the prompts and task head across clients with the same dual-channel affinity rule used in pre-training, adding no aggregation hyperparameter beyond the exploration coefficient $\kappa$.

During fine-tuning, each client selects prompts via Eq.~\eqref{eq:ucb}, synthesizes channel-specific prompts and applies the modulation gate Eq.~\eqref{eq:prompt-mod}, and trains locally by minimizing the downstream task loss $\mathcal{L}_{\text{task}}$. After local training, the client uploads task-head parameters, prompt-module parameters, and three small centroid sets $\{C_{k}^{I},C_{k}^{T},C_{k}^{\text{topo}}\}$ summarizing the prompts realized on the local data. The server computes channel-wise affinities $A_{ij}^{m}$ for $m\!\in\!\{I,T,\text{topo}\}$ via Eq.~\eqref{eq:affinity} applied to those centroid sets, aggregates each prompt group with its channel weights as in Eq.~\eqref{eq:aggregation}, and synchronizes the shared pool $\mathcal{P}$ by affinity-weighted averaging. This preserves channel-specific personalization while enabling federated knowledge transfer in the prompt space, reusing the affinity machinery from pre-training.

%% file: sections/theory.tex
\section{Theoretical Analysis}
\label{sec:theory}

Our theoretical analysis explains why ours channel-wise affinity-aware aggregation and local pre-training objective are effective. Theorem~\ref{thm:dual-channel} proves dual-channel aggregation strictly dominates single-channel schemes when feature and graph client similarities disagree, formalizing \S\ref{sec:cross-channel}. Theorem~\ref{thm:local-consistency} shows the local pre-training objective bounds cross-modal and semantic-structural discrepancy, grounding \S\ref{sec:empirical}. Theorem~\ref{thm:convergence} establishes channel-wise affinity weighting preserves the standard nonconvex federated convergence rate, with affinity estimation error as a non-accumulating additive term.

For client $i$, let $\Theta_{i,c}^{\star}$ denote the client-specific optimum of the parameter group associated with channel $c\in\{f,g\}$, and define the per-channel aggregation mismatch $\mathcal{E}_{i}^{c}(w)=\|\sum_{j=1}^{K}w_{j}\Theta_{j,c}^{\star}-\Theta_{i,c}^{\star}\|^{2}$ with channel-optimal weights $w_{i}^{c,\star}=\arg\min_{w\in\Delta^{K}}\mathcal{E}_{i}^{c}(w)$, where $\Delta^{K}$ is the probability simplex. The \emph{channel inconsistency} is then
\begin{equation}
\label{eq:channel-incons}
\delta_{i} \;=\; \bigl\|w_{i}^{f,\star} - w_{i}^{g,\star}\bigr\|^{2},
\end{equation}
which is precisely the cross-channel inconsistency illustrated in \S\ref{sec:cross-channel}: when a client's feature- and graph-channel similarities disagree, the cross-channel correlation $\rho$ is low and $\delta_{i}$ is large, exactly what the dual-channel design is built to exploit.

\begin{theorem}[Dual-channel dominance under channel inconsistency]
\label{thm:dual-channel}
Suppose each $\mathcal{E}_{i}^{c}$ is $\mu$-strongly convex on the simplex for $c\in\{f,g\}$. Under Assumptions A1--A4 with $\varepsilon=\max(\varepsilon_{f},\varepsilon_{g})$, the FedGAMMA dual-channel aggregation mismatch on the multimodal-fusion and graph-fusion parameter groups satisfies
\begin{equation}
\label{eq:dual-bound}
\mathcal{E}_{i}^{\text{dual}}
\;\le\;
\mathcal{E}_{i}^{f}(w_{i}^{f,\star}) \;+\; \mathcal{E}_{i}^{g}(w_{i}^{g,\star}) \;+\; 2\varepsilon,
\end{equation}
whereas any single-channel scheme that uses a shared weight $w$ obeys
\begin{equation}
\label{eq:single-bound}
\inf_{w\in\Delta^{K}}\!\Bigl(\mathcal{E}_{i}^{f}(w) + \mathcal{E}_{i}^{g}(w)\Bigr)
\;\ge\;
\mathcal{E}_{i}^{f}(w_{i}^{f,\star}) \;+\; \mathcal{E}_{i}^{g}(w_{i}^{g,\star}) \;+\; \tfrac{\mu}{2}\,\delta_{i}.
\end{equation}
Dual-channel aggregation therefore strictly dominates the optimal single-channel aggregation by at least $\tfrac{\mu}{2}\delta_{i}-2\varepsilon$, which is positive whenever $\delta_{i}>4\varepsilon/\mu$.
\end{theorem}

\textbf{Proof sketch.} The dual-channel scheme computes separate optimal weights $w_{i}^{f,\star}$ and $w_{i}^{g,\star}$ for each channel, incurring total residual $\mathcal{E}_{i}^{f}(w_{i}^{f,\star})+\mathcal{E}_{i}^{g}(w_{i}^{g,\star})$ plus an affinity estimation error $2\varepsilon$. By contrast, any single-channel scheme using a shared $w$ must minimize $\mathcal{E}_{i}^{f}(w)+\mathcal{E}_{i}^{g}(w)$. Strong convexity of each $\mathcal{E}_{i}^{c}$ over the simplex implies that the sum is lower-bounded by the channel-optimal residuals plus $\frac{\mu}{2}\|w-w_{i}^{f,\star}\|^{2}+\frac{\mu}{2}\|w-w_{i}^{g,\star}\|^{2}\geq\frac{\mu}{4}\|w_{i}^{f,\star}-w_{i}^{g,\star}\|^{2}=\frac{\mu}{4}\delta_{i}$, yielding the gap $\frac{\mu}{2}\delta_{i}-2\varepsilon$ which is strictly positive whenever the channel inconsistency $\delta_{i}$ exceeds $4\varepsilon/\mu$.

\begin{theorem}[Local pre-training controls cross-modal and structural discrepancy]
\label{thm:local-consistency}
For client $k$, let $\Delta_{mm}^{(k)}=\|\mathbf{S}^{I,(k)}-\mathbf{S}^{T,(k)}\|_{F}^{2}$ measure the cross-modal discrepancy between image-shared and text-shared representations, and $\Delta_{ss}^{(k)}=\|\bar{\mathbf{H}}^{(k)}-\mathbf{H}_s^{(k)}\|_{F}^{2}$ measure the semantic-structural discrepancy between the semantic-average and structural-role embeddings, where $\bar{\mathbf{H}}^{(k)}=\tfrac{1}{2}(\mathbf{H}_r^{(k)}+\mathbf{H}_k^{(k)})$. Suppose there exist $\mu_{1},\mu_{2}>0$ such that the alignment loss satisfies $\mathcal{L}_{\text{align}}^{(k)}\ge\mu_{1}\Delta_{mm}^{(k)}$ and the contrastive and regularization losses satisfy $\mathcal{L}_{\text{nce}}^{(k)}+\mathcal{L}_{\text{reg}}^{\text{g},(k)}\ge\mu_{2}\Delta_{ss}^{(k)}$. Then
\begin{equation}
\Delta_{mm}^{(k)}+\Delta_{ss}^{(k)}\le\frac{1}{\mu_{1}}\mathcal{F}_{k}(\Theta)+\frac{1}{\mu_{2}}\mathcal{F}_{k}(\Theta).
\end{equation}
Consequently, $\Delta_{mm}^{(k)}\le\mathcal{F}_{k}(\Theta)/\mu_{1}$ and $\Delta_{ss}^{(k)}\le\mathcal{F}_{k}(\Theta)/\mu_{2}$ individually, confirming that the same local objective that drives pre-training also bounds both sources of representational degradation diagnosed in \S\ref{sec:empirical}.
\end{theorem}

\textbf{Proof sketch.} The alignment loss $\mathcal{L}_{\text{align}}^{(k)}$ is a nonnegative principal term of the local pre-training objective $\mathcal{F}_{k}$, so $\mathcal{L}_{\text{align}}^{(k)}\le\mathcal{F}_{k}(\Theta)$; similarly $\mathcal{L}_{\text{nce}}^{(k)}+\mathcal{L}_{\text{reg}}^{\text{g},(k)}\le\mathcal{F}_{k}(\Theta)$. Substituting into the assumed lower bounds $\mathcal{L}_{\text{align}}^{(k)}\ge\mu_{1}\Delta_{mm}^{(k)}$ and $\mathcal{L}_{\text{nce}}^{(k)}+\mathcal{L}_{\text{reg}}^{\text{g},(k)}\ge\mu_{2}\Delta_{ss}^{(k)}$ yields $\Delta_{mm}^{(k)}\le\mathcal{F}_{k}(\Theta)/\mu_{1}$ and $\Delta_{ss}^{(k)}\le\mathcal{F}_{k}(\Theta)/\mu_{2}$. Hence minimizing $\mathcal{F}_{k}$ jointly suppresses both the cross-modal gap between image and text streams and the semantic-structural gap between feature-driven and topology-driven representations, directly supporting the empirical observations in \S\ref{sec:empirical}.

\begin{theorem}[Nonconvex convergence of FedGAMMA]
\label{thm:convergence}
Under Assumptions A1--A4 with $\varepsilon=\max(\varepsilon_{f},\varepsilon_{g})$, let each client perform $\tau$ local SGD steps per round with stepsize $\gamma\le 1/(4L\tau)$. Define $G_{t}=\mathbb{E}[\|\nabla\mathcal{F}(\Theta_{t})\|^{2}]$. The per-round descent under $L$-smoothness decomposes as
\begin{multline}
\mathbb{E}[\mathcal{F}(\Theta_{t+1})]-\mathcal{F}(\Theta_{t})
\le -\frac{\gamma\tau}{2}G_{t}
   + c_{1}\gamma^{2}L\tau\sigma^{2} \\
   + c_{2}\gamma^{3}L^{2}\tau^{2}(\sigma^{2}+\zeta^{2})
   + c_{3}\varepsilon,
\label{eq:per-round}
\end{multline}
where the four terms respectively capture gradient descent, stochastic noise, client drift, and affinity approximation error. Summing over $t=0,\dots,T-1$, telescoping, and rearranging yields
\begin{multline}
\frac{1}{T}\sum_{t=0}^{T-1}G_{t}
\le \frac{2\Delta_{\mathcal{F}}}{\gamma\tau T}
   + C_{1}\gamma L\sigma^{2}
   + C_{2}\gamma^{2}L^{2}\tau(\sigma^{2}+\zeta^{2}) \\
   + C_{3}\varepsilon,
\label{eq:convergence}
\end{multline}
where $\Delta_{\mathcal{F}}=\mathcal{F}(\Theta_{0})-\mathcal{F}_{\inf}$ and $C_{1,2,3}=O(1)$. The first three terms correspond to the standard FedAvg bound~\cite{mcmahan2017fedavg}; the $O(\varepsilon)$ term isolates the cost of channel-wise affinity approximation, which does not accumulate over rounds. With $\gamma=O((\tau T)^{-1/2})$, the three dominant terms balance to $O(1/\sqrt{T})$:
\begin{equation}
\label{eq:convergence-rate}
\frac{1}{T}\sum_{t=0}^{T-1}G_{t} = O(1/\sqrt{T}) + O(\varepsilon).
\end{equation}
\end{theorem}

\textbf{Proof sketch.} Starting from $L$-smoothness, one round of $\tau$ local SGD steps with channel-wise aggregation satisfies $\mathbb{E}[\mathcal{F}(\Theta_{t+1})]\le\mathcal{F}(\Theta_{t})-\frac{\gamma\tau}{2}G_{t}+O(\gamma^{2}L\tau\sigma^{2})+O(\gamma^{3}L^{2}\tau^{2}(\sigma^{2}+\zeta^{2}))+O(\varepsilon)$, where the $O(\varepsilon)$ term arises from substituting $w_{i}^{f},w_{i}^{g}$ for $w_{i}^{f,\star},w_{i}^{g,\star}$ and is bounded by Assumption A4. Summing over $t$ and telescoping yields the per-round bound. With $\gamma=O((\tau T)^{-1/2})$, the terms balance to $O(1/\sqrt{T})+O(\varepsilon)$.

%% file: sections/experiments.tex
\section{Experiments}
\label{sec:experiments}

In this section, we conduct extensive experiments on multimodal-attributed graph (MAG) federation datasets to evaluate FedGAMMA. The experimental study is organized to answer the following research questions:

\noindent\textbf{Q1}: Does FedGAMMA outperform federated multimodal graph baselines, federated graph foundation models, and federated adaptations of centralized multimodal graph foundation models across diverse downstream tasks? (~\S\ref{subsec:q1})

\noindent\textbf{Q2}: Does FedGAMMA preserve its advantage under low-resource scenarios? (~\S\ref{subsec:q2})

\noindent\textbf{Q3}: Which components of FedGAMMA are responsible for its empirical gains? (~\S\ref{subsec:q3})

\noindent\textbf{Q4}: How sensitive is FedGAMMA to its principal hyperparameters, in particular the global prompt-pool size and the modality-centroid dimensionality? (~\S\ref{subsec:q4})

\noindent\textbf{Q5}: Is the accuracy advantage of FedGAMMA achieved at an acceptable efficiency cost? (~\S\ref{subsec:q5})

\subsection{Experimental Setup}
\label{subsec:setup}

\noindent\textbf{Datasets.}\quad We evaluate FedGAMMA on twelve MAG datasets from the MM-OpenFGL benchmark~\cite{li2025mmopenfgl} spanning six downstream tasks. Table~\ref{tab:datasets} summarizes their statistics. Each node is associated with both image and text modalities encoded by frozen CLIP~\cite{radford2021clip}, while edges encode relational dependencies.

\begin{table}[t]
\centering
\caption{Dataset statistics. Datasets with `-' in the Labels column are evaluated on tasks that do not use node labels.}
\label{tab:datasets}
\renewcommand{\arraystretch}{1.15}
\resizebox{\linewidth}{!}{
\setlength{\tabcolsep}{6pt}
\begin{tabular}{ccccc}
\toprule
\textbf{Dataset} & \textbf{Nodes} & \textbf{Edges} & \textbf{Labels} & \textbf{Domain} \\
\midrule
Bili Food   & 1,579 & 6,544 & - & Video Rec. \\
Bili Music  & 6,038 & 21,592 & - & Video Rec. \\
DY          & 8,299 & 35,627 & - & Video Rec. \\
Ele Fashion & 97,766 & 199,602 & 12 & E-Commerce \\
Flickr30K   & 31,783 & 181,151 & - & Image Networks \\
Grocery     & 17,074 & 171,340 & 20 & E-Commerce \\
KU          & 5,370 & 22,052 & - & Video Rec. \\
Movies      & 16,672 & 218,390 & 20 & E-Commerce \\
QB          & 6,121 & 24,145 & - & Video Rec. \\
RedditS     & 15,894 & 566,160 & 20 & Social Media \\
Sports      & 50,250 & 356,202 & - & E-Commerce \\
Toys        & 20,695 & 126,886 & 18 & E-Commerce \\
\bottomrule
\end{tabular}
}
\end{table}

\noindent\textbf{Federated Simulation.}\quad To reflect realistic decentralized deployments, each dataset is partitioned across $5$ clients. Unless otherwise stated, the main results adopt the Louvain community partition~\cite{blondel2008louvain}, which preserves local topological coherence and induces realistic cross-client heterogeneity. The few-shot and sensitivity evaluations adopt the Metis partition~\cite{karypis1998metis} to test robustness under a different non-IID scenario.

\noindent\textbf{Baselines.}\quad We compare FedGAMMA against three orthogonal categories of baselines.
\textbf{(i)} Federated adaptations of multimodal graph encoders, namely Fed-MGNet~\cite{mgnet}, Fed-MMGCN~\cite{wei2019mmgcn}, and Fed-MGAT~\cite{tao2020mgat}, which extend representative multimodal GNNs to the federated setting via FedAvg~\cite{mcmahan2017fedavg}.
\textbf{(ii)} Multimodal extensions of federated graph foundation models, denoted with the \emph{MM-} prefix, namely MM-FedGFM~\cite{zhu2025fedgfm}, MM-FedBook~\cite{wu2025fedbook}, and MM-FedGALA~\cite{li2025fedgala}, which lift unimodal FedGFMs to MAGs by feeding fused multimodal node features into their original pipelines.
\textbf{(iii)} Federated adaptations of centralized multimodal graph foundation models, denoted with the \emph{Fed-} prefix, namely Fed-Mario~\cite{sun2026mario}, Fed-PLANET~\cite{li2025planet}, and Fed-UniGraph2~\cite{he2025unigraph2}, which decentralize centralized MGFMs through aggregation.

\noindent\textbf{Hyperparameter Settings.}\quad All methods use frozen CLIP~\cite{radford2021clip} image and text encoders with an output dimension of $d=768$ to extract node-level modality features. FedGAMMA stacks the cross-modal semantic enhancer, the topology-aware graph fusion module, and a $2$-layer GNN backbone with hidden dimension $768$. The number of centroids per client per channel is fixed to $M=5$. During fine-tuning, the global prompt pool size is set to $K=10$ with $k=5$ prompts selected per round. All baselines use their originally reported architectures with input dimension matched to CLIP. We tune all hyperparameters on a held-out validation split.

\noindent\textbf{Evaluation Metrics.}\quad We report task-specific metrics: accuracy (ACC) for node classification, AUC-ROC for link prediction and modality matching, Recall@5 for modality retrieval, CLIP-Score for modality alignment, and ROUGE-L scaled by $100$ for graph-to-text generation. Here CLIP-Score is the cosine similarity between the two modalities in the CLIP embedding space, and ROUGE-L measures the longest-common-subsequence overlap between generated and reference text. Higher values indicate better performance on all six tasks. For the generation task, the node representations produced by each pre-trained encoder condition a fine-tuned T5-large decoder~\cite{raffel2020t5}, a large pretrained text-to-text Transformer, so that all methods share an identical generation backbone and differ only in the graph representations they provide. All numbers are reported as mean$\pm$standard deviation over $10$ random seeds, except graph-to-text generation, which is averaged over $5$ seeds owing to the high cost.

Our experimental design follows established federated multimodal graph learning benchmark~\cite{li2025mmopenfgl}. The cross-modal semantic enhancer builds on multimodal deep learning principles~\cite{ngiam2011multimodal}, with text features encoded by Sentence-BERT-style language models~\cite{reimers2019sentencebert}. Graph self-supervised learning techniques~\cite{you2020graphmae,you2020graphcl,van2017vqvae} and personalized federated optimization methods~\cite{li2020fedprox,li2021ditto} inform the broader design of the pre-training and fine-tuning stages.

\subsection{Performance Comparison (Answer for Q1)}
\label{subsec:q1}

\begin{table*}[t]
    \setlength{\abovecaptionskip}{0.2cm}
    \renewcommand{\arraystretch}{2.3}
    \caption{\textbf{Performance comparison} between FedGAMMA and $9$ baselines spanning $3$ categories on $12$ MAG datasets. The globally best and second-best results in each column are highlighted in \textcolor{red}{\textbf{Red}} and \textcolor{blue}{\textbf{Blue}}.}
    \footnotesize
    \label{tab:main_results}
    \resizebox{\linewidth}{!}{
    \setlength{\tabcolsep}{1.0mm}{
    \renewcommand{\cellsdpad}{1.2ex}
    \renewcommand{\cellsdvskip}{1.8mm}
    \begin{tabular}{c|cc|cc|cc|cc|cc|cc}
        \toprule[1pt]
        \multirow{2}{*}{\textbf{Method}} & \multicolumn{2}{c|}{\textbf{Node Cls.\ (ACC)}} & \multicolumn{2}{c|}{\textbf{Link Pred.\ (AUC)}} & \multicolumn{2}{c|}{\textbf{Modal Match (AUC)}} & \multicolumn{2}{c|}{\textbf{Modal Retrieval (R@5)}} & \multicolumn{2}{c|}{\textbf{Modal Alignment (CLIP)}} & \multicolumn{2}{c}{\textbf{Modal Gen.\ (ROUGE-L)}} \\
        & Movies & Ele-fashion & Bili Music & Sports & KU & RedditS & DY & Grocery & Bili Food & QB & Flickr30K & Toys \\
        \midrule[0.1pt]

        Fed-MGNet
        & \cellsd{41.98}{0.63} & \cellsd{71.91}{0.59}
        & \cellsd{75.51}{0.83} & \cellsd{83.26}{0.74}
        & \cellsd{71.77}{0.63} & \cellsd{70.22}{0.10}
        & \cellsd{43.15}{1.41} & \cellsd{46.23}{0.87}
        & \cellsd{57.34}{1.14} & \cellsd{45.89}{0.63}
        & \cellsd{54.97}{1.01} & \cellsd{17.84}{0.75}
        \\
        Fed-MMGCN
        & \cellsd{41.84}{0.41} & \cellsd{72.67}{1.00}
        & \cellsd{75.63}{0.34} & \cellsd{86.12}{0.77}
        & \cellsd{72.37}{1.05} & \cellsd{74.06}{0.93}
        & \cellsd{44.72}{1.13} & \cellsd{48.35}{0.76}
        & \cellsd{59.51}{1.08} & \cellsd{44.17}{0.31}
        & \cellsd{55.91}{0.84} & \cellsd{17.77}{0.44}
        \\
        Fed-MGAT
        & \cellsd{44.31}{0.62} & \cellsd{73.45}{0.36}
        & \cellsd{75.48}{0.41} & \cellsd{88.22}{0.69}
        & \cellsd{77.92}{0.99} & \cellsd{74.95}{0.92}
        & \cellsd{47.28}{0.44} & \cellsd{53.62}{0.05}
        & \cellsd{61.88}{0.97} & \cellsd{49.34}{0.84}
        & \cellsd{55.56}{0.93} & \cellsd{18.72}{0.15}
        \\
        \midrule[0.1pt]

        MM-FedGFM
        & \cellsd{46.70}{0.93} & \cellsd{72.66}{0.41}
        & \cellsd{72.06}{0.94} & \cellsd{80.43}{0.96}
        & \cellsd{69.80}{0.51} & \cellsd{70.22}{0.18}
        & \cellsd{49.37}{2.53} & \cellsd{56.84}{1.12}
        & \cellsd{60.51}{1.05} & \cellsd{67.23}{0.98}
        & \cellsd{55.84}{1.16} & \cellsd{\textcolor{blue}{19.07}}{\textcolor{blue}{0.51}}
        \\
        MM-FedBook
        & \cellsd{46.64}{0.93} & \cellsd{72.64}{1.52}
        & \cellsd{72.44}{1.68} & \cellsd{80.47}{0.80}
        & \cellsd{70.08}{0.17} & \cellsd{70.22}{0.24}
        & \cellsd{47.93}{0.76} & \cellsd{53.71}{0.94}
        & \cellsd{68.62}{0.37} & \cellsd{65.44}{0.21}
        & \cellsd{54.89}{1.20} & \cellsd{18.48}{0.13}
        \\
        MM-FedGALA
        & \cellsd{\textcolor{blue}{48.78}}{\textcolor{blue}{0.61}} & \cellsd{\textcolor{blue}{75.28}}{\textcolor{blue}{0.56}}
        & \cellsd{\textcolor{blue}{76.29}}{\textcolor{blue}{0.83}} & \cellsd{80.59}{0.10}
        & \cellsd{72.62}{0.82} & \cellsd{71.99}{0.30}
        & \cellsd{49.53}{0.72} & \cellsd{54.28}{0.83}
        & \cellsd{68.73}{0.62} & \cellsd{69.87}{0.14}
        & \cellsd{54.92}{1.02} & \cellsd{18.13}{0.51}
        \\
        \midrule[0.1pt]

        Fed-Mario
        & \cellsd{48.10}{1.01} & \cellsd{75.13}{0.43}
        & \cellsd{73.76}{0.06} & \cellsd{88.79}{0.14}
        & \cellsd{\textcolor{blue}{80.47}}{\textcolor{blue}{1.07}} & \cellsd{\textcolor{blue}{76.03}}{\textcolor{blue}{0.42}}
        & \cellsd{46.42}{0.88} & \cellsd{54.35}{1.05}
        & \cellsd{\textcolor{blue}{70.95}}{\textcolor{blue}{0.84}} & \cellsd{63.11}{0.39}
        & \cellsd{55.46}{0.98} & \cellsd{17.82}{0.90}
        \\
        Fed-PLANET
        & \cellsd{47.60}{0.72} & \cellsd{74.84}{0.70}
        & \cellsd{69.83}{0.29} & \cellsd{\textcolor{blue}{89.98}}{\textcolor{blue}{0.04}}
        & \cellsd{78.30}{0.21} & \cellsd{73.69}{1.02}
        & \cellsd{\textcolor{blue}{54.16}}{\textcolor{blue}{3.27}} & \cellsd{\textcolor{blue}{61.83}}{\textcolor{blue}{0.59}}
        & \cellsd{67.15}{0.96} & \cellsd{\textcolor{blue}{72.51}}{\textcolor{blue}{0.13}}
        & \cellsd{\textcolor{blue}{56.35}}{\textcolor{blue}{0.87}} & \cellsd{18.84}{0.70}
        \\
        Fed-UniGraph2
        & \cellsd{43.20}{0.02} & \cellsd{72.60}{0.04}
        & \cellsd{68.78}{0.76} & \cellsd{82.15}{0.79}
        & \cellsd{74.29}{0.51} & \cellsd{70.43}{0.76}
        & \cellsd{43.89}{0.15} & \cellsd{49.47}{0.51}
        & \cellsd{63.62}{0.28} & \cellsd{60.75}{0.73}
        & \cellsd{56.11}{1.12} & \cellsd{18.42}{0.15}
        \\
        \midrule[0.1pt]

        \textbf{FedGAMMA (Ours)}
        & \cellsd{\textcolor{red}{\textbf{52.69}}}{\textcolor{red}{\textbf{0.67}}} & \cellsd{\textcolor{red}{\textbf{78.27}}}{\textcolor{red}{\textbf{0.07}}}
        & \cellsd{\textcolor{red}{\textbf{77.01}}}{\textcolor{red}{\textbf{0.47}}} & \cellsd{\textcolor{red}{\textbf{91.96}}}{\textcolor{red}{\textbf{0.12}}}
        & \cellsd{\textcolor{red}{\textbf{84.73}}}{\textcolor{red}{\textbf{0.74}}} & \cellsd{\textcolor{red}{\textbf{79.43}}}{\textcolor{red}{\textbf{0.21}}}
        & \cellsd{\textcolor{red}{\textbf{64.73}}}{\textcolor{red}{\textbf{0.81}}} & \cellsd{\textcolor{red}{\textbf{71.24}}}{\textcolor{red}{\textbf{1.34}}}
        & \cellsd{\textcolor{red}{\textbf{83.91}}}{\textcolor{red}{\textbf{0.65}}} & \cellsd{\textcolor{red}{\textbf{74.68}}}{\textcolor{red}{\textbf{1.27}}}
        & \cellsd{\textcolor{red}{\textbf{56.42}}}{\textcolor{red}{\textbf{1.03}}} & \cellsd{\textcolor{red}{\textbf{19.42}}}{\textcolor{red}{\textbf{0.45}}}
        \\
        \bottomrule[1pt]
    \end{tabular}
    }}
\end{table*}

FedGAMMA attains the best result on all twelve cells, with the largest margins on cross-modal tasks. As Table~\ref{tab:main_results} shows for the nine-baseline comparison on twelve datasets under the Louvain partition with $5$ clients, it reaches $84.73\%$ on KU modality matching, $64.73\%$ on DY retrieval, and $83.91\%$ on Bili Food alignment. The generative task is the only place where the lead narrows: the shared fine-tuned T5-large decoder dominates generation quality, so all methods fall within a band of at most $1.65$ ROUGE-L points, far tighter than the spread on any discriminative task, and FedGAMMA still ranks first on both datasets with $56.42$ on Flickr30K and $19.42$ on Toys.

\noindent \textbf{Comparison to Federated Multimodal Graph Encoders.}
Fed-MGNet, Fed-MMGCN, and Fed-MGAT fuse modalities through simple concatenation or gating and aggregate with uniform FedAvg, so this shallow fusion erases modality-specific complementarity and their cross-modal performance collapses, confirming the modality collapse diagnosed in \S~\ref{sec:empirical}. Specifically, Fed-MGAT reaches only $47.28\%$ on DY retrieval and $49.34\%$ on QB alignment, far below FedGAMMA's $64.73\%$ and $74.68\%$, yet its Toys generation score of $18.72$ trails FedGAMMA by merely $0.70$ ROUGE-L points, showing that the decoder-dominated task masks encoder differences. FedGAMMA avoids this collapse through shared-private decoupling and dual-channel affinity aggregation.

\noindent \textbf{Comparison to Multimodal-Extended FedGFMs.}
MM-FedGFM and MM-FedBook trail FedGAMMA by up to $15\%$ on modality matching because unimodal-style fusion discards modality-specific information before pre-training, although MM-FedGFM still posts the second-best Toys generation score of $19.07$ on the decoder-dominated task. MM-FedGALA leads this category, ranking second on both node classification datasets and on Bili Music link prediction, yet it trails FedGAMMA by $15.20\%$ on DY and $16.96\%$ on Grocery. This asymmetry stems from single-channel aggregation, which cannot serve node-level and cross-modal objectives simultaneously. FedGAMMA resolves the trade-off through optimal-transport-based alignment and dual-channel aggregation.

\noindent \textbf{Comparison to Federated Adaptations of Centralized MGFMs.}
Fed-Mario and Fed-PLANET remain competitive on structural tasks but degrade on cross-modal ones, exemplified by Fed-PLANET's $54.16\%$ on DY retrieval against FedGAMMA's $64.73\%$. Fed-UniGraph2 is more balanced across discriminative tasks yet still trails FedGAMMA on every one of them, for instance $74.29\%$ against $84.73\%$ on KU modality matching, while its Flickr30K generation score of $56.11$ sits only $0.31$ ROUGE-L points behind, echoing the decoder saturation that compresses all methods on the generative task. These centralized MGFMs are designed for full graph corpora, and naive FedAvg leaves their locally aligned semantic geometry inconsistent across clients with no server-side mechanism to reconcile the differences. FedGAMMA closes this gap with centroid-based affinity estimation, which aligns feature-channel and graph-channel similarity without accessing raw data.

\subsection{Few-shot Learning Evaluation (Answer for Q2)}
\label{subsec:q2}

\begin{table}[t]
    \setlength{\abovecaptionskip}{0.2cm}
    \renewcommand{\arraystretch}{1.5}
    \caption{\textbf{$2$-shot results} of FedGAMMA and $6$ competitive baselines on $3$ representative MAG tasks.}
    \small
    \label{tab:fewshot_results}
    \resizebox{\linewidth}{!}{
    \setlength{\tabcolsep}{1.5mm}{
    \begin{tabular}{c|c|c|c}
        \toprule[1pt]
        \textbf{Method} & \makecell{\textbf{Node Cls.}\\\textbf{(Movies)}} & \makecell{\textbf{Modal Match}\\\textbf{(KU)}} & \makecell{\textbf{Modal Align.}\\\textbf{(QB)}} \\
        \midrule[0.1pt]
        MM-FedGFM     & $13.80{\scriptstyle\pm}1.20$ & $70.97{\scriptstyle\pm}0.32$ & $58.15{\scriptstyle\pm}0.24$ \\
        MM-FedBook    & $15.30{\scriptstyle\pm}1.74$  & $70.20{\scriptstyle\pm}0.70$ & $59.38{\scriptstyle\pm}0.12$ \\
        MM-FedGALA    & $18.15{\scriptstyle\pm}1.05$  & $73.74{\scriptstyle\pm}0.38$ & $59.17{\scriptstyle\pm}0.45$ \\
        \midrule[0.1pt]
        Fed-Mario     & $17.56{\scriptstyle\pm}0.24$ & $76.73{\scriptstyle\pm}0.87$ & $57.92{\scriptstyle\pm}0.61$ \\
        Fed-PLANET    & $18.79{\scriptstyle\pm}0.04$ & $71.29{\scriptstyle\pm}0.64$ & $61.72{\scriptstyle\pm}0.28$ \\
        Fed-UniGraph2 & $10.37{\scriptstyle\pm}1.24$ & $70.25{\scriptstyle\pm}1.18$ & $53.85{\scriptstyle\pm}0.89$ \\
        \midrule[0.1pt]
        \textbf{FedGAMMA (Ours)} & $\textcolor{red}{\mathbf{18.95}}{\scriptstyle\pm}\textcolor{red}{\mathbf{1.05}}$ & $\textcolor{red}{\mathbf{77.22}}{\scriptstyle\pm}\textcolor{red}{\mathbf{0.77}}$ & $\textcolor{red}{\mathbf{67.43}}{\scriptstyle\pm}\textcolor{red}{\mathbf{0.91}}$ \\
        \bottomrule[1pt]
    \end{tabular}
    }}
\end{table}

FedGAMMA leads on all three few-shot tasks, confirming that its advantage comes from pre-training quality rather than from abundant supervision. In this $2$-shot setting under the Metis partition with $5$ clients (Table~\ref{tab:fewshot_results}), only two labeled examples per class leave the task head unable to compensate for degraded representations, so downstream performance reflects pre-training quality directly. On Movies node classification FedGAMMA achieves $18.95\%$, marginally ahead of Fed-PLANET at $18.79\%$. On KU modality matching it reaches $77.22\%$, ahead of the strongest baseline Fed-Mario at $76.73\%$. The lead widens most on QB alignment, where its $67.43\%$ exceeds the second-best Fed-PLANET by $5.71\%$.

The baseline pattern shows how label scarcity amplifies design deficiencies. Node classification is the hardest setting for every method, and MM-FedGFM, MM-FedBook, and Fed-UniGraph2 fall to $13.80\%$, $15.30\%$, and $10.37\%$ respectively, because two labels per class cannot recover the information that unimodal-style fusion or naive aggregation already discards during pre-training. A sharper trend appears across the three task types. FedGAMMA's lead widens monotonically with how strongly a task depends on the pre-trained cross-modal geometry rather than on the task head, growing from $0.16\%$ on node classification to $0.49\%$ on modality matching and $5.71\%$ on modality alignment, which confirms that the few-shot setting probes representation quality more directly as the task becomes more cross-modal. The baselines further trade off between the two cross-modal tasks. Fed-Mario reaches the strongest matching score at $76.73\%$ on KU yet only $57.92\%$ on QB alignment, while Fed-PLANET inverts this with the best baseline alignment at $61.72\%$ on QB but a weaker $71.29\%$ on KU. This instability is the signature of single-objective pre-training that sharpens one cross-modal relation at the expense of the other, whereas FedGAMMA leads both because its shared-private enhancer and dual-channel aggregation preserve a balanced cross-modal geometry. Routing feature-channel and graph-channel updates through separate affinity matrices keeps that balance intact even when downstream labels are scarce.

\subsection{Ablation Study (Answer for Q3)}
\label{subsec:q3}

\begin{table}[t]
    \setlength{\abovecaptionskip}{0.2cm}
    \renewcommand{\arraystretch}{1.5}
    \caption{\textbf{Ablation study} of FedGAMMA on $3$ representative tasks.}
    \small
    \label{tab:ablation_results}
    \resizebox{\linewidth}{!}{
    \setlength{\tabcolsep}{1.5mm}{
    \begin{tabular}{l|c|c|c}
        \toprule[1pt]
        \textbf{Variant} & \makecell{\textbf{Link Pred.}\\\textbf{(Bili Music)}} & \makecell{\textbf{Modal Retrieval}\\\textbf{(Grocery)}} & \makecell{\textbf{Modal Gen.}\\\textbf{(Flickr30K)}} \\
        \midrule[0.1pt]
        w/o Modal Fusion         & $69.01{\scriptstyle\pm}0.43$ & $58.27{\scriptstyle\pm}0.85$ & $55.31{\scriptstyle\pm}0.84$ \\
        w/o View Router          & $73.65{\scriptstyle\pm}0.49$ & $69.84{\scriptstyle\pm}0.63$ & $55.97{\scriptstyle\pm}0.52$ \\
        w/o Global Prompt Pool   & $71.29{\scriptstyle\pm}0.60$ & $70.52{\scriptstyle\pm}0.97$ & $55.78{\scriptstyle\pm}0.67$ \\
        \midrule[0.1pt]
        \textbf{FedGAMMA}        & $\textcolor{red}{\mathbf{77.01}}{\scriptstyle\pm}\textcolor{red}{\mathbf{0.47}}$ & $\textcolor{red}{\mathbf{71.24}}{\scriptstyle\pm}\textcolor{red}{\mathbf{1.34}}$ & $\textcolor{red}{\mathbf{56.42}}{\scriptstyle\pm}\textcolor{red}{\mathbf{1.03}}$ \\
        \bottomrule[1pt]
    \end{tabular}
    }}
\end{table}

The cross-modal fusion module is by far the most important component, while the global prompt pool and the view router contribute comparable, second-order gains. Table~\ref{tab:ablation_results} isolates the contribution of these three components under the Louvain partition with $5$ clients across link prediction, modality retrieval, and generation. The \textit{w/o Modal Fusion} variant removes the cross-modal semantic enhancer and causes the largest drop, with Grocery retrieval falling from $71.24\%$ to $58.27\%$ and Bili Music link prediction from $77.01\%$ to $69.01\%$. Removing the global prompt pool or the view router is far less damaging, with the largest such drop being $5.72\%$ on Bili Music, while Flickr30K generation stays within $1.11$ ROUGE-L points of the full model across all variants.

The degradation pattern shows a design hierarchy with task-specific roles. The cross-modal enhancer is foundational, and removing it hurts modality retrieval most, with Grocery losing $12.97\%$ against Bili Music's $8.00\%$, because retrieval is a purely cross-modal task while link prediction can still exploit the topology that survives without shared-private decoupling. This localizes the enhancer's contribution to exactly the cross-modal geometry diagnosed as modality collapse in ~\S\ref{sec:empirical}. The global prompt pool and the view router instead play second-order roles that concentrate on link prediction, where removing them costs $5.72\%$ and $3.36\%$ on Bili Music yet at most $1.40\%$ on Grocery, since both inject cross-client structural priors that topology-driven link prediction relies on while retrieval is already carried by the enhancer. That the prompt pool outweighs the view router on this task also explains its slightly larger overall contribution. The nearly flat generation column, within $1.11$ ROUGE-L points across all variants, confirms that every component acts at the representation level, an effect the strong fine-tuned decoder masks under full supervision and that resurfaces only when the decoder is data-starved in the few-shot scenario of Q2.

\begin{table}[t]
    \setlength{\abovecaptionskip}{0.2cm}
    \renewcommand{\arraystretch}{1.2}
    \caption{\textbf{Efficiency profiling} on Movies under node classification.}
    \footnotesize
    \label{tab:complexity}
    \centering
    \resizebox{\linewidth}{!}{
    \setlength{\tabcolsep}{1.2mm}{
    \begin{tabular}{l|ccc|ccc}
        \toprule[1pt]
        \multirow{2}{*}{\textbf{Method}} & \multicolumn{3}{c|}{\textbf{Pre-training}} & \multicolumn{3}{c}{\textbf{Fine-tuning}} \\
        & Time & Mem. & Comm. & Time & Mem. & Comm. \\
        \midrule[0.1pt]
        MM-FedGFM      & $79.6$  & $4.04$  & $2082.6$    & $22.6$   & $4.06$  & $416.5$    \\
        MM-FedBook     & $84.2$  & $5.99$  & $2403.7$  & $23.7$   & $6.07$  & $480.7$   \\
        MM-FedGALA     & $86.2$  & $5.90$  & $2187.1$    & $26.3$   & $6.09$  & $717.6$    \\
        \midrule[0.1pt]
        Fed-Mario      & $82.3$  & $4.79$  & $1744.3$   & $23.4$   & $6.25$  & $548.9$   \\
        Fed-PLANET     & $88.7$  & $17.24$ & $2303.7$ & $24.2$  & $21.08$ & $2060.7$  \\
        Fed-UniGraph2  & $91.1$  & $2.65$  & $1965.8$    & $21.9$  & $10.27$  & $713.2$    \\
        \midrule[0.1pt]
        \textbf{FedGAMMA (Ours)} & $91.5$ & $10.97$ & $2150.1$ & $22.4$ & $5.35$ & $725.0$ \\
        \bottomrule[1pt]
    \end{tabular}
    }}
\end{table}

\subsection{Sensitivity Analysis (Answer for Q4)}
\label{subsec:q4}

\begin{figure*}[t]
    \centering
    \includegraphics[width=\linewidth]{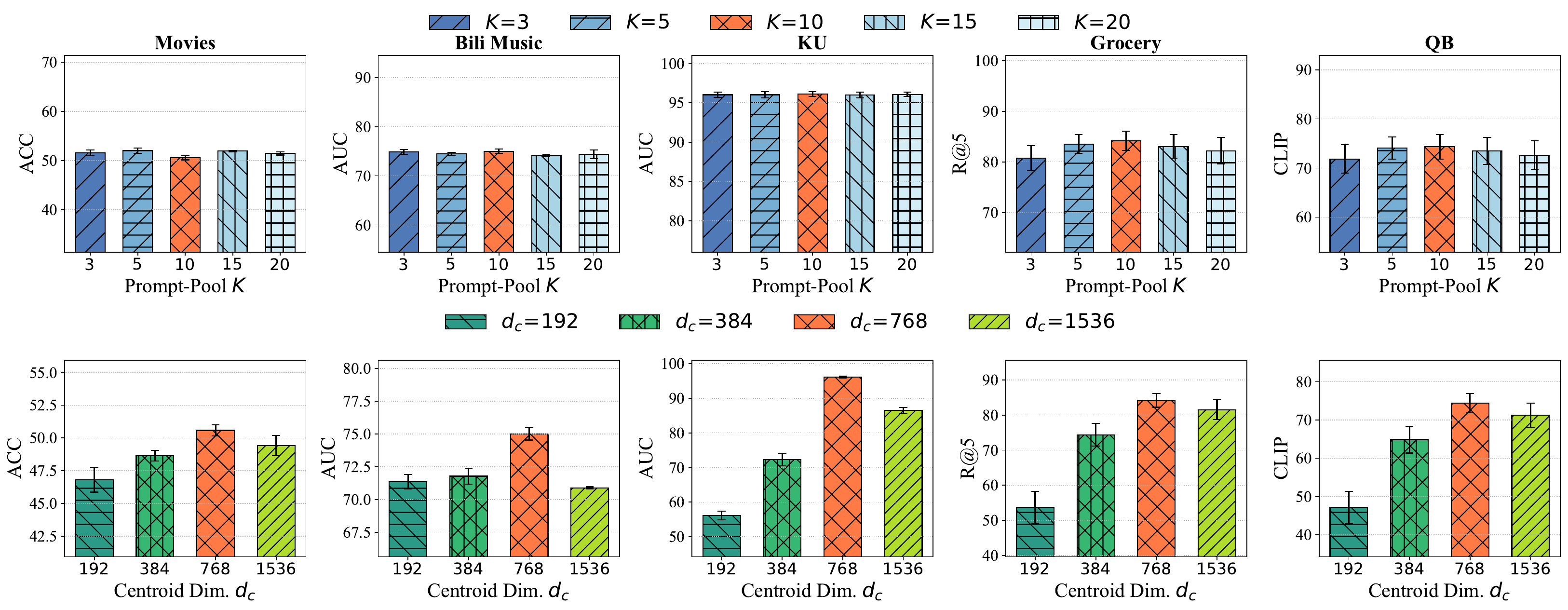}
    \caption{\textbf{Sensitivity analysis} of FedGAMMA on five MAG tasks. The top row sweeps the global prompt-pool size $K$, and the bottom row sweeps the modality-centroid dimensionality $d_c$. Whiskers denote $\pm 1$\,std, and the red star marks the default configuration.}
    \label{fig:sensitivity}
\end{figure*}

FedGAMMA is highly robust to the global prompt-pool size $K$, while the modality-centroid dimensionality $d_c$ has a clear unimodal optimum that coincides with the CLIP feature dimension. Fig.~\ref{fig:sensitivity} sweeps $K\in\{3,5,10,15,20\}$ and $d_c\in\{192,384,768,1536\}$ on five MAG tasks under the Metis partition with $5$ clients, averaging each configuration over $3$ seeds with defaults $K{=}10$ and $d_c{=}768$. Performance varies by at most $3.4\%$ across $K$ on Grocery retrieval, whereas $d_c$ peaks at $d_c=768$, with shrinking to $d_c=192$ costing roughly $40\%$ on KU modal match and inflating to $d_c=1536$ costing about $10\%$ on KU.

The insensitivity to $K$ reflects a structural property of the federated setting, since with $5$ clients the number of distinct heterogeneity scenarios is small and $K=3$ already covers them. The residual sensitivity that does appear is concentrated on the cross-modal tasks, since $K$ moves Grocery retrieval by $3.4\%$ and QB alignment by $2.5\%$ but KU matching by only $0.1\%$, mirroring the Q3 finding that the prompt pool mainly refines cross-client structural priors. The $d_c$ profile is more fundamental and far more task-dependent. Node classification and link prediction shift by roughly $4\%$ across the entire sweep, whereas the cross-modal tasks swing by $27\%$ to $40\%$, because $d_c$ controls the spectral fidelity of the centroid descriptor that drives affinity-aware aggregation, and cross-modal matching is its most demanding consumer. At $d_c=192$ the centroid loses the spectral information needed for fine-grained modality distinctions, while at $d_c=1536$ the descriptor exceeds the intrinsic dimensionality and introduces estimation variance. The peak at $d_c=768$, the CLIP feature dimension, confirms that this hyperparameter is determined by the representation space rather than by costly search.

\begin{figure*}[t]
    \centering
    \includegraphics[width=\linewidth]{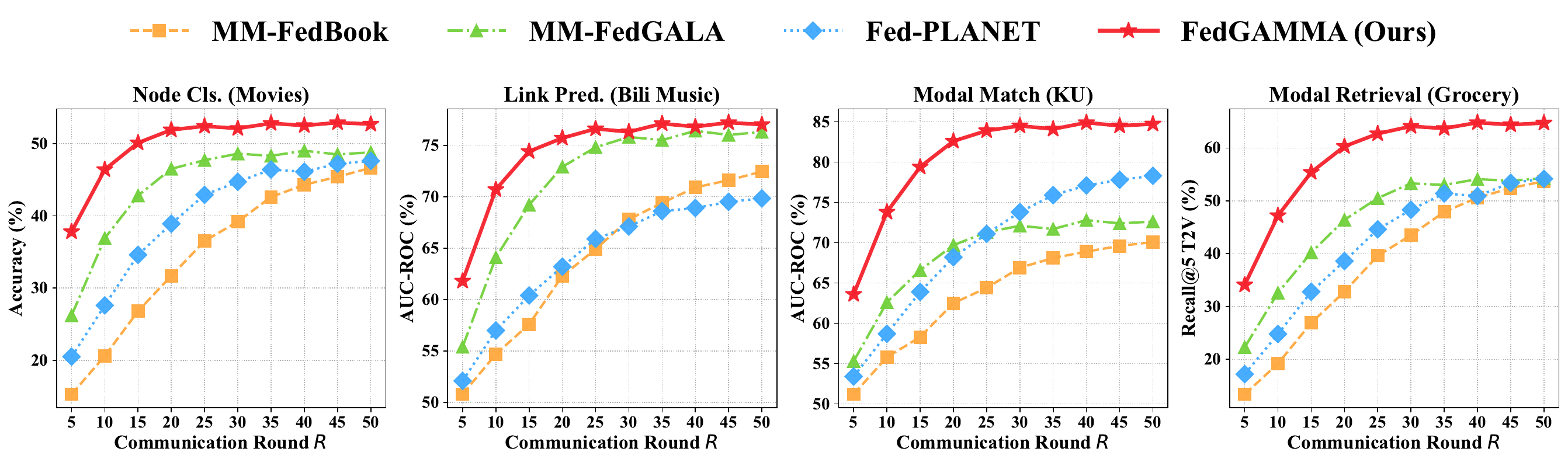}
    \caption{\textbf{Convergence analysis} of pre-training. Each panel reports the validation metric of one task at $R=5,10,\dots,50$ communication rounds.}
    \label{fig:convergence}
\end{figure*}

\subsection{Efficiency and Convergence Analysis (Answer for Q5)}
\label{subsec:q5}

FedGAMMA's per-round cost is comparable to the FedGFM baselines, and it converges in $2$-$5\times$ fewer communication rounds, so its end-to-end budget is no larger while it alone reaches the best accuracy on all twelve datasets. Table~\ref{tab:complexity} profiles the per-stage cost on Movies node classification, reporting wall-clock time in seconds, CUDA peak reserved memory in gigabytes, and communication payload in megabytes, and Fig.~\ref{fig:convergence} tracks validation metrics from $R=5$ to $R=50$, both under the standard protocol of $10$ local epochs per round. On cost, FedGAMMA stays within the same scenario as the FedGFM baselines, taking $91.5$s and exchanging $2.15$ GB during pre-training, comparable to Fed-PLANET at $88.7$s and $2.30$ GB, while its $10.97$ GB of peak memory sits well below Fed-PLANET's $17.24$ GB despite the richer cross-modal enhancer and three-view graph fusion. Its fine-tuning stage is lightweight at $22.4$s, $5.35$ GB, and $725.0$ MB of payload, a roughly $3\times$ communication reduction from frozen-backbone prompt tuning. On convergence, FedGAMMA starts from the best early-round representation and plateaus by roughly $R=25$, matching the round-$50$ performance of the strongest baseline with $2$-$5\times$ fewer communication rounds, whereas MM-FedGALA flattens early at a lower ceiling and MM-FedBook is still climbing at $R=50$.

These two views reconcile into a single efficiency profile. The distinct convergence shapes trace back to the mechanisms isolated in Q1 and Q3. MM-FedBook climbs slowly because one fused feature impoverishes its early-round representations, MM-FedGALA plateaus once its single shared channel saturates, and Fed-PLANET suffers conflicting gradients from naively averaged heterogeneous updates. FedGAMMA avoids all three failure modes by routing feature-channel and graph-channel updates through separate affinity matrices, yielding both the fastest early ascent and the highest plateau. Its heavier memory footprint is confined to the one-time pre-training stage and falls back to $5.35$ GB during deployment-time adaptation. Consequently, the accuracy advantage of FedGAMMA incurs no additional end-to-end cost, which is the efficiency profile a practical federated foundation model demands.

%% file: sections/conclusion.tex
\section{Conclusion}
\label{sec:conclusion}

We studied federated multimodal graph foundation models, which jointly requires multimodal interaction, graph-structural grounding, privacy-preserving collaboration, and client heterogeneity. We proposed FedGAMMA, a two-stage framework of federated pre-training and prompt-based fine-tuning. It integrates shared-private semantic enhancement, topology-aware graph fusion, and dual-channel affinity-aware aggregation, with parameter-efficient adaptation via graph-aware prompts.


This work also has limitations that point to future directions. The framework focuses on image-text graphs, and extending it to richer modalities is a natural next step. Although affinity-aware aggregation is privacy-preserving at the raw-data level, a formal privacy analysis and stronger theory remain valuable. Finally, validation in larger real-world environments, with more adaptive prompt-sharing for extreme client heterogeneity, would further strengthen its practical case.